\documentclass[a4paper,fleqn]{dc}
\usepackage[authoryear]{natbib}
\usepackage{url}

\usepackage{graphicx}
\usepackage{subfigure}
\usepackage{amsmath,amssymb,amsfonts}
\usepackage{algorithmic}
\usepackage{color}
\usepackage{bm}
\usepackage{algorithm}
\usepackage{ntheorem}
\usepackage{extarrows}
\usepackage{array}
\usepackage{multirow}
\usepackage{verbatim}
\usepackage{threeparttable}
\usepackage{placeins}
\usepackage{mdframed}
\usepackage{lipsum}
\usepackage{pifont}
\usepackage{bbding}
\usepackage{adjustbox}
\usepackage{makecell}
\usepackage{orcidlink}
\usepackage{siunitx}
\usepackage{tabularx}
\usepackage{float}
\usepackage{graphicx}
\usepackage{subcaption}
\usepackage{hyperref}
\newtheorem{theorem}{Theorem}

\newtheorem*{proof}{Proof}

\def\tsc#1{\csdef{#1}{\textsc{\lowercase{#1}}\xspace}}
\tsc{WGM}
\tsc{QE}
\tsc{EP}
\tsc{PMS}
\tsc{BEC}
\tsc{DE}

\usepackage{xcolor}
\definecolor{lightpink}{RGB}{255,240,245} 
\usepackage{placeins}
\usepackage{eso-pic}

\makeatletter
\let\orig@makecaption\@makecaption

\renewcommand{\fnum@table}{\sffamily\bfseries\tablename~\thetable}
\newcommand{\setcaptionwidth}[1]{%
  \long\def\@makecaption##1##2{%
    \vskip\abovecaptionskip
    \centering
    \small 
    \parbox{#1\linewidth}{%
      ##1\\
      \sffamily\mdseries ##2\par}
    \vskip\belowcaptionskip}%
}
\makeatother

\usepackage{xpatch}

\makeatletter
\xpatchcmd{\ps@pprintTitle}{20}{23}{}{}
\makeatother

\usepackage{fancyhdr}
\begin{document}
\begin{sloppypar}
\thispagestyle{empty}

\let\WriteBookmarks\relax
\def\floatpagepagefraction{1}
\def\textpagefraction{.001}

\shorttitle{Spike-driven Large Language Model}
\shortauthors{Han Xu et~al.}

\title[mode=title]{Spike-driven Large Language Model}

\author[1,2,3]{\textcolor{black}{Han Xu} }
\fnmark[1]
\author[1,4,5]{\textcolor{black}{Xuerui Qiu} }
\fnmark[1]
\author[1,6]{\textcolor{black}{Baiyu Chen} }
\author[7]{\textcolor{black}{Xinhao Luo} }
\author[1,3]{\textcolor{black}{Xingrun Xing} }
\author[1,2]{\textcolor{black}{Jiahong Zhang}}
\author[3]{\textcolor{black}{Bo Lei} }
\author[3,8]{\textcolor{black}{Tiejun Huang}  }
\author[1,2]{\textcolor{black}{Bo Xu} }
\cormark[1]
\author[1,2,9,10]{\textcolor{black}{Guoqi Li} }
\cormark[1]

\fntext[fn1]{These authors contributed equally to this work.}
\cortext[cor1]{Corresponding authors. E-mails: guoqi.li@ia.ac.cn; xubo@ia.ac.cn}

\address[1]{Institute of Automation, Chinese Academy of Sciences, Beijing 100190, China}
\address[2]{School of Artificial Intelligence, University of Chinese Academy of Sciences, Beijing 100049, China}
\address[3]{Brain-Inspired Models Group, Beijing Academy of Artificial Intelligence, Beijing 100862, China}
\address[4]{School of Future Technology, University of Chinese Academy of Sciences, Beijing 101408, China}
\address[5]{Zhongguancun Academy, Beijing 100094, China}
\address[6]{Pengcheng Laboratory, Guangdong 518055, China}
\address[7]{School of Advanced Inter-disciplinary Sciences, University of Chinese Academy of Sciences, Beijing 101408, China}
\address[8]{School of Computer Science, Peking University, Beijing 100871, China}
\address[9]{Key Laboratory of Brain Cognition and Brain-inspired Intelligence Technology, Beijing 100190, China;}
\address[10]{Spiking Intelligence Lab, Tianqiao \& Chirssy Chen Institute, Shanghai 201203, China}

\begin{abstract}
Current Large Language Models (LLMs) are primarily based on large-scale dense matrix multiplications. Inspired by the brain's information processing mechanism, we explore the fundamental question: how to effectively integrate the brain's spiking-driven characteristics into LLM inference. Spiking Neural Networks (SNNs) possess spike-driven characteristics, and some works have attempted to combine SNNs with Transformers. However, achieving spike-driven LLMs with billions of parameters, relying solely on sparse additions, remains a challenge in the SNN field.
To address the issues of limited representational capacity and sparsity in existing spike encoding schemes at the LLM level, we propose SDLLM, a spike-driven large language model that eliminates dense matrix multiplications through sparse addition operations.
Specifically, we use the plug-and-play $\gamma$-SQP two-step spike encoding method to ensure that the quantization process aligns with the model's semantic space, mitigating representation degradation caused by binary spikes.
Furthermore, we introduce bidirectional encoding under symmetric quantization and membrane potential clipping mechanisms, leading to spike trains with no or low firing counts dominating, significantly reducing the model's spike firing rate, while halving the number of time steps. Experimental results show that SDLLM not only significantly reduces inference costs 
but also achieves state-of-the-art task performance under the spike-based paradigm.
For example, compared to previous spike-based LLMs, SDLLM reduces energy consumption by 7 $\times$ and improves accuracy by 4.2\%. Our model provides inspiration for the architecture design of the next generation of event-driven neuromorphic chips.
\end{abstract}


\maketitle

\renewcommand{\thefootnote}{}
\renewcommand{\thefootnote}{\fnsymbol{footnote}}

\section{Introduction}
Large language models (LLMs) have become a major breakthrough in artificial intelligence due to their strong performance on general language tasks \citep{touvron2023llama,touvron2023llama2,zhang2022opt}. However, their reliance on large-scale dense matrix multiplications leads to significant computational and energy challenges during deployment, particularly on resource-constrained devices \citep{shao2023omniquant}. In contrast, the human brain performs complex tasks efficiently with less than 20 watts of power \citep{Balasubramaniana2021BrainP}. Inspired by this, we ask a fundamental question: how can spike-driven mechanisms be incorporated into large-model inference?
Spiking neural networks (SNNs), inspired by the brain, employ sparse spike-driven communication and synaptic connections \citep{yao2024spike,maass1997networks}, executing only sparse additions (1 for firing, 0 otherwise). When combined with neuromorphic hardware, they form the neuromorphic computing paradigm, offering a promising solution for efficient modeling \citep{frenkel2023bottom,schuman2022opportunities}.

Several studies have explored combining SNNs with Transformer architectures.
(i) Although encouraging results have been reported on small-scale vision tasks such as SFA \citep{yao2025scaling,luo2024integer,liu2025efficient}, these tasks are far simpler than LLMs in both scale and complexity and therefore place substantially lower demands on spike encoding.
(ii) Models including SpikingBERT \citep{lv2023spikebert}, SpikeGPT \citep{zhu2023spikegpt}, and SpikeLM \citep{xing2023spikelm} mainly evaluate small-scale supervised tasks, whereas large-scale LLM tasks are considerably more complex and impose much higher requirements on spike encoding in terms of time steps, sparsity, and representational capacity, making direct transfer from vision or small supervised settings infeasible.
(iii) SpikeLLM \citep{xing2024spikellm} represents an early attempt to introduce spikes into LLMs; however, it relies on multi-timestep integer quantization to enhance representational capacity and does not provide a feasible spike encoding scheme aligned with neuromorphic spike-driven computation.
Consequently, achieving purely spike-driven inference for billion-scale language models using only sparse additions remains an open challenge in the SNN community.

In this work, we aim to build a spike-driven LLM that achieves LLM-level modeling capabilities through a novel spike encoding mechanism. By jointly optimizing spike representational capacity and sparsity, our model maintains competitive performance while enabling efficient inference.
To narrow the performance gap between ANNs and SNNs, we propose a two-step spike encoding method based on the $\gamma$-Semantic Quantization Principle ($\gamma$-SQP). The $\gamma$-SQP indicates that the quantization process should follow the model’s semantic space learned from training, thereby reducing the concentration of quantization cost on key semantic dimensions, thus reducing spike quantization error. Thus, guided by $\gamma$-SQP, membrane potentials are mapped to integer spike counts with lower semantic quantization error and temporally unfolded into binary spike trains, enabling high-performance spike-driven inference.

Controlling the spike firing rate is critical, as it directly affects the computational load and efficiency of neuromorphic SNNs. We therefore analyze firing-rate patterns under the two-step spiking conversion and introduce two sparsity strategies at the spike encoding level to significantly reduce the model's firing rate.
First, bidirectional spike encoding with symmetric quantization reduces the occurrence of high spike counts while halving temporal unfolding steps. Second, a ReLU-based truncation mechanism increases the proportion of zero-valued spike counts, combined with sparse rotation matrices to alleviate truncation-induced quantization error.

We evaluated SDLLM on multiple models, including LLaMA2-7B, LLaMA2-13B, LLaMA3-8B, and Qwen2.5-14B \citep{qwen2.5}, and assessed it on various tasks. These tasks include commonsense reasoning (PIQA \citep{bisk2020piqa}, ARC-easy \citep{clark2018think}, ARC-challenge \citep{clark2018think}, HellaSwag \citep{clark2018think}, and WinoGrande \citep{sakaguchi2021winogrande}), and a broader and more challenging range of tasks (reading comprehension tasks: BoolQ \citep{clark2019boolq} and SQuAD \citep{rajpurkar2016squad}; world knowledge tasks: TriviaQA \citep{touvron2023llama}; mathematical reasoning tasks: GSM8K \citep{cobbe2021gsm8k}), to demonstrate SDLLM’s broad adaptability in a multi-task environment and its effectiveness across models. Our contributions are summarized as follows:

\begin{table*}
\centering
\setcaptionwidth{1} 
\caption{Methodological spike encoding detailed methodological comparison of the differences between QuaRot, SpikeLLM, and SDLLM.}
\label{tab:method_comparison}
\setlength{\tabcolsep}{3pt}
\vspace{0.5em}
\begin{adjustbox} {width=\linewidth} 
\begin{tabular}{l|c|c|c}
Dimension & QuaRot & SpikeLLM & SDLLM \\
\midrule
Quantization scheme 
& Low-bit INT quantization 
& Multi-step integer spike encoding 
& $\gamma$-SQP two-step spike encoding \\

Spike representation 
& --
& Integer-valued spikes 
& Binary / ternary spikes \\

Matrix multiplication scheme 
& \ding{51} 
& \ding{51} 
& \ding{55} (spike-driven sparse addition) \\

Spike firing mode
& -- 
& Synchronous 
& Synchronous/Asynchronous \\

Explicit firing-rate regulation 
& -- 
& \ding{55}  
& \ding{51} (with components) \\

Neuromorphic compatibility 
& \ding{55} 
& \ding{55} 
& \ding{51} \\
\end{tabular}
\end{adjustbox}
\end{table*}

\begin{itemize}
\item We propose a sparse-addition-based spike-driven LLM that achieves LLM-level modeling capability via a novel spike encoding mechanism and eliminates dense matrix multiplications during inference.

\item We propose a plug-and-play two-step spike encoding with $\gamma$-SQP that mitigates binary-spike-induced representational degradation by aligning the quantization process with the model’s semantic space and adopting stepwise count-to-spike encoding, supporting both synchronous (multi-step latency) and asynchronous (single-step latency) spike firing modes.

\item We analyze firing patterns in count-to-spike expansion and introduce two sparsity mechanisms--bidirectional encoding under symmetric quantization and membrane potential clipping--leading to zero-valued, low-valued spike counts dominating, while halving temporal unfolding, significantly reducing the model's firing rate.

\item Experimental results demonstrate that our SNN model significantly reduces inference cost (lower FLOPs and up to 13× energy consumption reduction over the INT4 baseline), while achieving task performance at the INT4 state-of-the-art (SOTA) level (significantly enhancing the representational capacity of event-driven LLMs and yielding higher performance on multiple tasks).

\item Our SNN model provides inspiration for the architecture design of the next generation of event-driven neuromorphic chip.
\end{itemize}

\begin{figure*}
    \centering
    \begin{minipage}[t]{0.82\linewidth}
        \includegraphics[width=\linewidth]{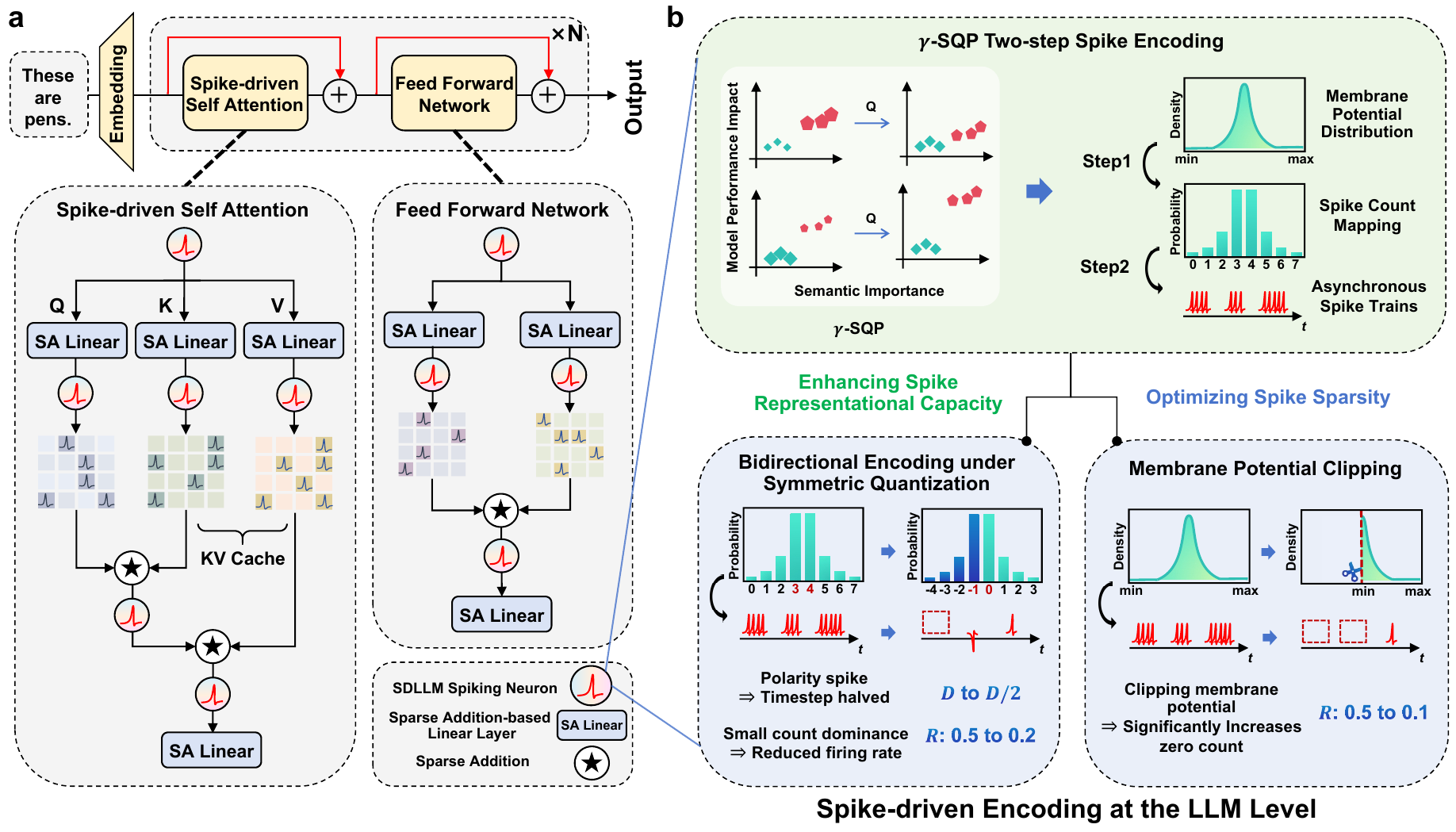}
    \end{minipage}%
    \hfill
    \begin{minipage}[t]{0.174\linewidth}
        \includegraphics[width=\linewidth]{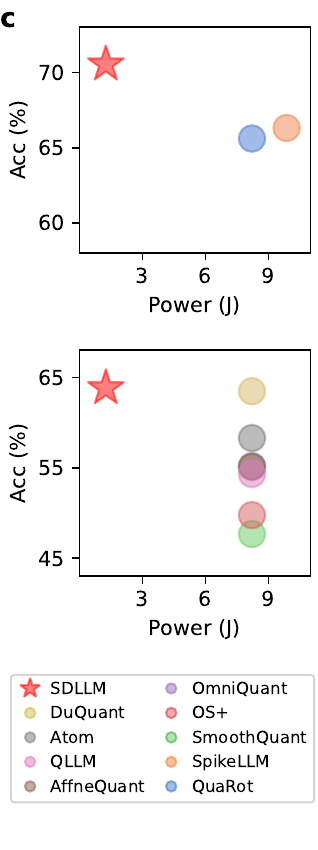}
    \end{minipage}
    \caption{(a) SDLLM replaces dense matrix multiplication with spike-driven sparse addition in Transformer-based LLMs through novel LLM-level spike encoding. (b) To address insufficient representation capacity and sparsity in existing LLM-level spike encoding, we propose $\gamma$-SQP two-step spike encoding to reduce semantic quantization loss and mitigate binary spike representation degradation, combined with bidirectional encoding and membrane potential clipping to significantly reduce firing rate (e.g., 0.5$\rightarrow$0.2, 0.5$\rightarrow$0.1) and halve temporal steps ($D$$\rightarrow$$D$/2). (c) Performance comparison of SDLLM vs. SpikeLLM (top) and low-bit edge solutions (bottom).}
    \label{fig:1}
\end{figure*}

\section{Related Work}
\paragraph{Training of Spiking Neural Networks.} The development of SNNs has long been hindered by the
challenge of training non-differentiable binary spikes. To address this, researchers have focused on improving training methods and architectural designs. Recently, two primary methods for training high-performance SNNs have emerged. One approach is to convert ANNs into spike form through neuron equivalence \citep{li2021free,hao2023reducing}, known as ANN-to-SNN conversion. However, this method requires long simulation time steps and increases energy consumption. 
Another line of methods directly trains SNNs using surrogate gradients \citep{wu2018spatio}, which typically incurs high training costs.
In contrast to previous work, we propose a plug-and-play two-step spike encoding method for inference, without the need to train the SNN. 

\paragraph{Spiking Neural Networks for Natural Language Processing.} As LLMs like GPT-3 scale, their rising computational and energy demands raise cost and sustainability concerns. To address this, SNNs are being explored in NLP for energy-efficient modeling. Bi-SNN \citep{xiao2022towards} introduced a bidirectional SNN for sentiment classification and translation. SpikingBERT \citep{lv2023spikebert,bal2024spikingbert} and SpikeLM \citep{xing2023spikelm} combined SNNs with BERT via spike-based distillation and dual-spike encoding, but remain limited to million-scale parameters and small supervised tasks. SpikeGPT \citep{zhu2023spikegpt} adopted binary spike activations and simplified attention to reduce computation, {yet still faces clear limitations in scaling and task complexity.} SpikeLLM \citep{xing2024spikellm} scaled SNNs to 7B-parameter Transformers using a “best-brain” framework, achieving competitive results; {however, its inference process relies on multi-timestep integer activations to compensate for performance, which weakens spike characteristics and makes it difficult to be compatible with neuromorphic hardware.}

\paragraph{Spiking Model Compression.} 
{Researchers have explored various compression techniques for spiking neural networks (SNNs) to reduce model cost, mainly including:}
(i) Sparsification in SNNs \citep{han2015learning,wei2025qp}, which typically adapts pruning techniques from traditional ANNs to suit both the spatial and temporal domains of spiking models~\citep{shi2023towards,shen2024efficient}. While effective on simple datasets and shallow networks, achieving strong performance on complex tasks and deeper architectures remains challenging. 
(ii) Knowledge distillation \citep{hinton2015distilling} transfers knowledge from large ANNs or SNNs into smaller SNNs to reduce model size and power consumption. However, numerous methods~\citep{takuya2021training,xu2023constructing} only distill final output logits, leading to incomplete knowledge transfer and limited effectiveness in downstream SNN performance.
(iii) Quantization~\citep{jacob2018quantization,krishnamoorthi2018quantizing}, which is particularly important for hardware deployment, reduces the bit-widths of weights and activations to enable more energy-efficient inference. Quantization methods mainly fall into two categories: \textit{post-training quantization} (PTQ) and \textit{quantization-aware training} (QAT). Recent studies on SNN quantization~\citep{qiu2025quantized} have primarily focused on small-scale vision tasks with convolutional and transformer architectures, typically adopting QAT to compress weights and constrain timesteps, relying on task-specific training protocols to achieve strong performance. However, such methods are not directly applicable to spike-based large language models.
In this study, we use a novel spike encoding method to compress SNNs. The $\gamma$-SQP two-step spike encoding significantly enhances spike representation capacity while balancing timestep length, and quantizes weights using RTN—the simplest PTQ method, without the need for training or calibration. At the same time, based on the firing patterns of the count-to-spike paradigm, we design two sparsity components—bidirectional encoding under symmetric quantization and membrane potential clipping—to regulate spike sparsity, significantly reducing the firing rate and halving the timestep length, supporting LLM-level spike encoding from both spike representation capacity and sparsity perspectives.

{\paragraph{Neuromorphic hardware.} Neuromorphic hardware is inspired by the structure and function of the biological brain, featuring compute–memory co-location and spike-driven event-based computation \citep{roy2019towards,schuman2022opportunities}. Existing platforms include both pure SNN architectures \citep{merolla2014million,davies2018loihi} and hybrid ANN/SNN designs \citep{Ma2022,Hoppner2021} to improve flexibility. Recent systems such as the asynchronous sensing–computing SoC Speck demonstrate ultra-low-power potential, with a static power of about 0.42 mW and 0.7–15 mW in typical neuromorphic vision tasks \citep{yao2024spike}. Meanwhile, Intel Loihi 2 significantly enhances programmability and representational flexibility over previous generations, supporting programmable neuron models with microcode-defined state transitions and firing logic, as well as graded spike events representing different polarities or amplitudes\citep{Intel2021}. Recent work has further deployed LLMs with 16-bit activations and 1.58-bit weights on multi-chip Loihi 2 systems using asynchronous processing, achieving up to 10× energy reduction and 4× throughput improvement over edge GPUs, highlighting the potential of neuromorphic hardware for large language models \citep{zhu2024scalable}. From an algorithm–hardware co-design perspective, our approach provides useful insights for next-generation neuromorphic chips and efficient AI systems.}

\section{Preliminary}
\paragraph{Quantization  Framework.} We employ uniform quantization for both weights and activations,
with the quantization scale determined by min–max statistics,
requiring no manual hyperparameter tuning.
For a full-precision matrix $X$, the $N$-bit quantization
process is defined as:
\begin{equation}
\label{eq:quantization}
\begin{aligned}
&\hat{{X}} =
\mathrm{Clamp}\!\left(
\left\lfloor \frac{{X}}{\Delta} \right\rceil + {Z},
\, 0,\,
2^N - 1
\right); \\
&\Delta = \frac{\max({X}) - \min({X})}{2^N - 1}, \quad
{Z} = -\left\lfloor \frac{\min({X})}{\Delta} \right\rceil.
\end{aligned}
\end{equation}
Here, $\hat{{X}}$ is the quantized
counterpart, $\Delta$ is the quantization step size, $\lfloor \cdot \rceil$ is the rounding function, and $Z$ represents the zero-point value. Moreover, \(\operatorname{Clamp}\{x, a, b\}\) confines \(x\) within range \([ a, b]\). The quantization process described above can be expressed using the quantization function $Q(\cdot)$.

\paragraph{LIF Spike Neuron.}  
The Leaky Integrate-and-Fire (LIF) neuron is a simplified biologically inspired model that simulates the electrical activity of biological neurons~\citep{roy2019towards}. It integrates incoming signals while accounting for the gradual decay (leakage) of membrane potential over time. When the membrane potential reaches a threshold, a spike is generated and the potential is reset to a baseline. Due to its balance between computational simplicity, efficiency, and biological plausibility, the LIF model is widely used in neuroscience and computational models to simulate neural information processing. The update process is defined as follows:
\begin{align}
    \mathbf{U}^\ell[t]&=\mathbf{H}^\ell[t-1]+f({\mathbf{W}^\ell},\mathbf{X}^{\ell-1}[t]), \label{eq:2}\\
    \mathbf{S}^\ell[t]&=\mathbf{\Theta}(\mathbf{U}^\ell[t]-\vartheta), \label{eq:3} \\
    \mathbf{H}^\ell[t]&= \beta\mathbf{U}^\ell[t]\cdot(1-\mathbf{S}^\ell[t])+\mathbf{V}_{reset}\cdot \mathbf{S}^\ell[t]. \label{eq:4}
\end{align}

Here, the membrane potential \( \mathbf{U}^\ell[t] \) at time step \( t \) is updated based on the previous potential \( \mathbf{H}^\ell[t-1] \) and the input signal \( f(\mathbf{W}^\ell, \mathbf{X}^{(\ell-1)}[t]) \), as shown in Eq.~\ref{eq:2}. A spike is triggered when the potential exceeds the threshold \( \vartheta \), with the step function \( \Theta \) in Eq.~\ref{eq:3} indicating the firing decision. If a spike occurs, the membrane potential is reset to \( \mathbf{V}_{\text{reset}} \), where \( \beta < 1 \) is the decay factor, as shown in Eq.~\ref{eq:4}.

\section{Methods}
In this section, we address the challenge of current spike encoding methods being inadequate for LLM-level modeling capabilities by focusing on both representational capacity and sparsity, as shown in Fig.~\ref{fig:1}. First, in Section \ref{sec:4.1}, we propose a two-step spike encoding method based on $\gamma$-SQP, which improves the representational capacity of binary spikes and reduces the performance gap between SNNs and ANNs. Next, in Section \ref{sec:4.2}, we analyze the firing rate patterns induced by the count-to-spike expansion and identify the main sources of spike redundancy. Building on this analysis, Sections \ref{sec:4.3} and \ref{sec:4.4} introduce two complementary sparsity strategies, including bidirectional encoding under symmetric quantization and membrane potential clipping, to suppress spike firing rates while preserving model accuracy. Together, these strategies enable SDLLM to replace dense matrix multiplications with spike-driven sparse additions, achieving more efficient inference in LLMs.

\subsection{Two-Step Spike Encoding with the $\gamma$-Semantic Quantization Principle
} 
\label{sec:4.1}
We aim to construct a spike-driven large language model (LLM) based on sparse addition. During inference, we replace traditional dense matrix multiplications with spike-based addition using binary spikes, thereby enabling a more efficient model. 
However, the expressive capacity of binary spike representations is insufficient, resulting in significant performance degradation, which highlights the need to improve both the fidelity and efficiency of spike-based representations.
Normalization scaling factors $\gamma$ align with the semantic representation space learned during training, reflecting the relative semantic importance of different feature dimensions. For key semantic directions, quantization perturbations are more likely to lead to a degradation in model accuracy.
To effectively reduce spike quantization errors, we propose a plug-and-play two-step spike encoding method with the $\gamma$-Semantic Quantization Principle ($\gamma$-SQP) to improve the performance of spiking neurons in LLM settings.
We first quantize continuous membrane potentials into integer-form spike counts via the $\gamma$-SQP; then, through time-domain expansion, these integer spikes are further mapped into 0/1 spike trains, enabling event-driven discrete computation.

\paragraph{Step One: Integer Spike Count Quantization with the $\gamma$-SQP.} 
Hadamard rotations can homogenize feature distributions, and QuaRot \citep{ashkboos2024quarot} leverages this property to mitigate outlier issues under low-bit quantization. The Walsh--Hadamard matrix is an efficient orthogonal transform; applying a $d \times d$ Hadamard transform $H_{\mathrm{d}}$ to $x \in \mathbb{R}^d$ incurs a computational complexity of only $O(d\log_2 d)$ ($\ll O(d^2)$).

Building on this, we observe that QuaRot exploits the computational invariance of RMSNorm in Transformer modules. Specifically, for a orthogonal matrix $\mathbf Q$, RMSNorm satisfies
$
\mathrm{RMSNorm}(\mathbf X)=\mathrm{RMSNorm}(\mathbf X\mathbf Q^\top)\mathbf Q 
$.
After fusing the RMSNorm layer with the subsequent linear layer, the computation can remain equivalent by applying corresponding orthogonal transformations to the activations and weights:
\begin{equation}
\frac{\mathbf X}{\mathrm{RMS}(\mathbf X)} \leftarrow \frac{\mathbf X\mathbf Q}{\mathrm{RMS}(\mathbf X)},\quad
\mathbf W \leftarrow \mathbf Q^\top \mathbf W\,\mathrm{Diag}(\gamma) .
\end{equation}

However, while this design accounts for the computational invariance of RMSNorm, it does not consider the semantic importance characterized by the per-dimension scaling parameter $\gamma$. Absorbing $\gamma$ into the weights during quantization disrupts its original role in modulating semantics in the representation space, leading to mismatched propagation of quantization perturbations along semantic directions.

Based on this observation, we propose the $\gamma$-Semantic Quantization Principle, which states that quantization and Hadamard mixing should follow the semantic space characterized by $\gamma$.
Low-bit quantization inevitably introduces quantization error. 
Following \citet{hassibi1992second} and \citet{xing2024spikellm}, we adopt a first-order differentiation based on the gradient of the quantization cost function, focusing on the local quantization sensitivity of the variable itself.
For a matrix $X$, we consider the error introduced by quantization on the variable itself and treat it as the quantization cost of the variable. 
Specifically, we define the quantization cost function as
$
\mathcal L(X)=\|X-Q(X)\|_2^2.
$
To analyze the effect of each dimension on the quantization cost, we consider a small perturbation applied to $X$, i.e., $X\rightarrow X+\delta X$, and define the induced cost variation as
$
\delta\mathcal L=\mathcal L(X+\delta X)-\mathcal L(X).
$
Using the first-order Taylor expansion, the above cost variation can be approximated as
$
\delta\mathcal L\approx \delta X^\top\nabla_X\mathcal L(X).
$
Based on the above first-order approximation, we define the quantization cost metric of the variable as
\begin{equation}
\boldsymbol{\kappa}(X)=X\circ X,
\end{equation}
where $\circ$ denotes the element-wise product.
Meanwhile, the per-dimension scaling parameter $\gamma$ in RMSNorm reflects the semantic importance of different representation dimensions in the network. 
When the quantization cost $\boldsymbol{\kappa}$ is concentrated on dimensions with large $\gamma$, corresponding to more semantically important directions, the resulting quantization perturbations are more likely to degrade model accuracy. 
Orthogonal mixing mitigates this issue by rotating the representation space, redistributing $\boldsymbol{\kappa}$ across dimensions and thereby reducing its concentration on semantically important directions.

\begin{theorem}
(Dispersion of Quantization Cost over the $\gamma$-Identified Semantic Subspace)
Let $\boldsymbol{\kappa}(X)$ denote the element-wise quantization cost of a variable $X$.
Let $\mathcal S_\gamma \subset [d]$ denote the semantic subspace identified by the RMSNorm scaling parameter $\gamma$, corresponding to the set of dimensions with large $\gamma$, and let $|\mathcal S_\gamma| = m$.
For a random orthogonal matrix  $\mathbf Q$, we have
\begin{equation}
\mathbb{E}_Q\!\left[\sum_{i \in \mathcal S_\gamma} \kappa_i(X\mathbf Q)\right]
= \frac{m}{d} \sum_{i=1}^d \kappa_i(X). \label{eq:gamma-dispersion}
\end{equation}
Random orthogonal mixing redistributes the quantization cost evenly across dimensions, preventing it from concentrating on the $\gamma$-identified semantic subspace.
\end{theorem}
\begin{proof}
Let $Z = X\mathbf Q$. Since $\mathbf Q$ is a random orthogonal matrix, $Z$ corresponds to an isotropic random rotation of the fixed vector $X$, while preserving its $\ell_2$ norm, i.e., $\|Z\|_2 = \|X\|_2$.
By rotational symmetry, the squared coordinates $Z_i^2$ are statistically identical for all $i \in [d]$, which implies
\begin{equation}
\mathbb{E}_Q[Z_i^2] = \frac{1}{d} \|X\|_2^2.
\end{equation}
Since the element-wise quantization cost $\kappa_i(\cdot)$ depends on the magnitude of the corresponding coordinate, the above result indicates that each dimension contributes equally to the total quantization cost after mixing.
Summing over the semantic subspace $\mathcal S_\gamma$ yields Equation~\eqref{eq:gamma-dispersion}.
\end{proof}

\begin{figure*}
\centering
\begin{minipage}[b]{0.49\linewidth}
  \centering
  \includegraphics[width=\linewidth]{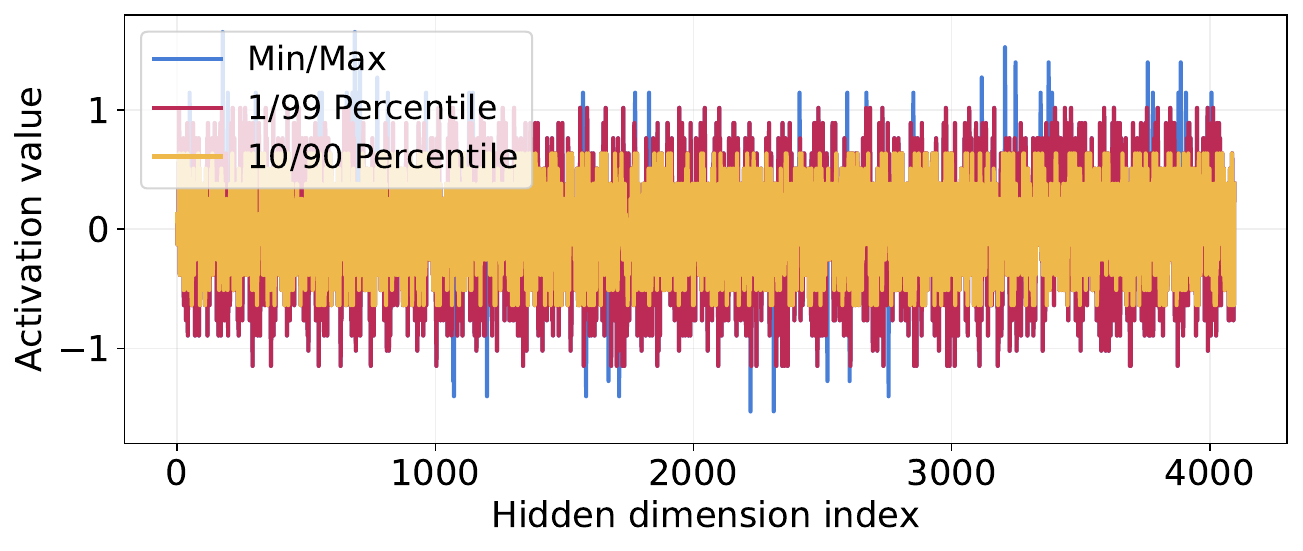}
  \scriptsize (a) QuaRot
\end{minipage}
\begin{minipage}[b]{0.49\linewidth}
  \centering
  \includegraphics[width=\linewidth]{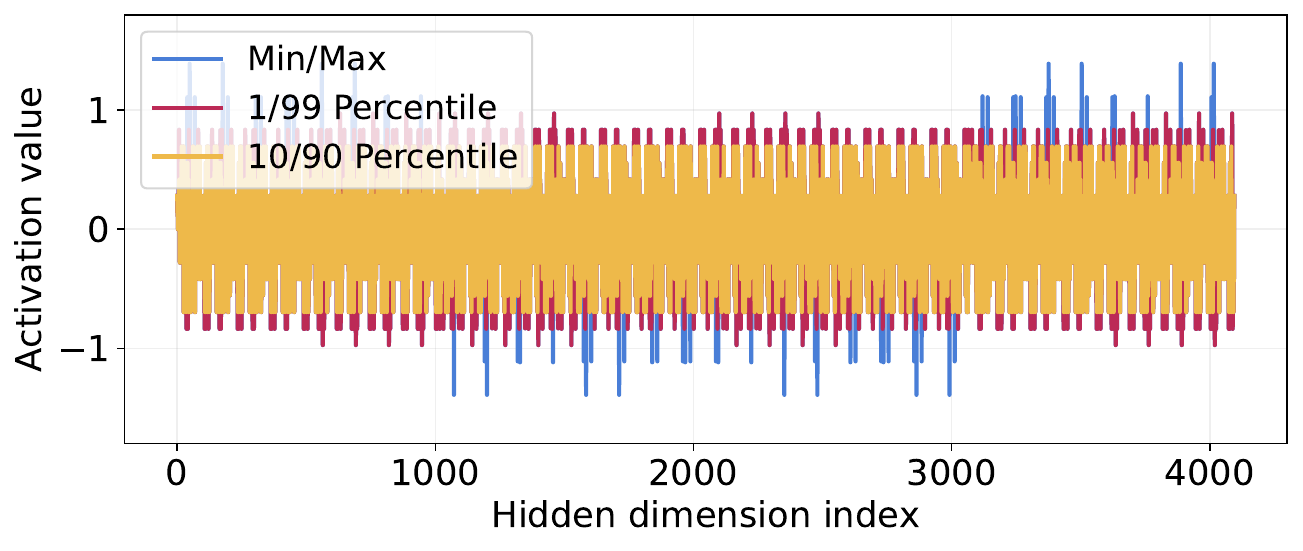}
  \scriptsize (b) SDLLM
\end{minipage}
\caption{Comparison of attention activation distributions before and after $\gamma$-semantic space alignment (QuaRot vs.\ SDLLM). Compared to QuaRot, SDLLM exhibits a more balanced and more compact attention activation distribution.
}
\label{fig:quantization_distribution}
\end{figure*}

\begin{figure*}
    \centering
    \includegraphics[width=0.62\linewidth, ]{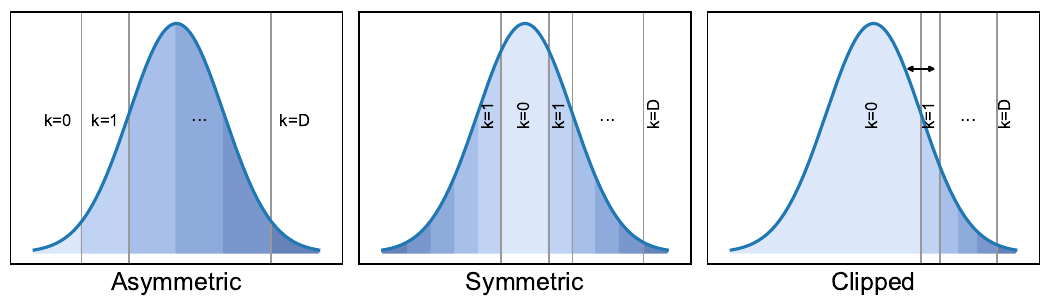}
    \caption{
    Different methods of spike quantization methods. Clipped method has adjustable 0-1 boundary, the other thresholds are uniformly split among 0 and saturation value D.
    }
    \label{fig:6}
\end{figure*}

\begin{figure}
    \centering
    \includegraphics[width=0.92\linewidth]{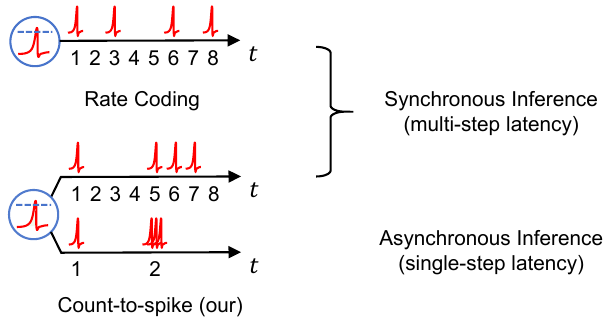}
    \caption{
    Unlike rate coding, our count-to-spike paradigm with $T \times D$ timesteps ($D$ as unfolded dimension) supports both synchronous firing (multi-step latency) and asynchronous firing (single-step latency) modes on neuromorphic chips.
    }
    \label{fig:asyn}
\end{figure}

On the one hand, since the quantization cost is positively correlated with the squared activation magnitude,
$
\kappa_i \propto x_i^2,
$
when the activation $x$ is numerically scaled by the per-dimension RMSNorm parameter $\gamma$, i.e., $\gamma \circ x$, certain dimensions exhibit larger magnitudes, leading to more pronounced quantization cost $\boldsymbol{\kappa}$ on those dimensions.
On the other hand, $\gamma$ is simultaneously a global parameter learned during training, which serves to identify the $\gamma$-semantic space formed in the representation space, i.e., the directions that are semantically more important to the model.
Based on this observation, the key lies in whether the directions emphasized by the numerical scaling of $\gamma$ are aligned with the semantic importance directions identified by $\gamma$.
Only when this alignment holds can random orthogonal mixing transfer the quantization cost $\boldsymbol{\kappa}$ highlighted by $\gamma \circ x$ from semantically important directions to less important ones, thereby reducing its concentration along semantic directions and alleviating the performance degradation under low-bit quantization.
In contrast, if the $\gamma$ learned in the representation space fails to correctly indicate the semantic structure of the weight space, such that the dimensions emphasized by $\mathbf W\mathrm{Diag}(\gamma)$ do not correspond to semantically important directions, random orthogonal mixing will instead drive the quantization cost to accumulate along truly important directions and amplify semantic perturbations.
This failure mode corresponds to the case introduced by the transformation $\mathbf Q^\top \mathbf W\mathrm{Diag}(\gamma)$. As shown in Fig.~\ref{fig:quantization_distribution}, by adhering to the correct $\gamma$-semantic space, SDLLM exhibits a more balanced and more compact distribution of attention activations compared to QuaRot, thereby alleviating low-bit quantization loss.

Accordingly, we reformulate the update of the LIF neuron in Eq.~\ref{eq:3} as
\begin{equation}
\mathbf{S}^\ell[t] = \gamma\text{-SQP}\left(\mathbf{U}^\ell[t],0,D\right). \label{eq:GMQuaRot_LIF}
\end{equation}

Here, we use $\gamma$-SQP to perform a Round-to-Nearest (RTN) quantization on the membrane potential after Hadamard rotation, i.e., $\mathbf{U} \leftarrow \mathbf{U}\,\mathrm{Diag}(\gamma)\mathbf H$ to match the correct $\gamma$ semantic space, yielding an integer-valued spike count $\mathbf{S}^\ell[t] \in \{0, \dots, D\}$ at layer $\ell$ and time step $t$.
For INT-$B$ spike count quantization, $D = 2^{B} - 1$.

\begin{figure}
    \centering
    \includegraphics[width=\linewidth]{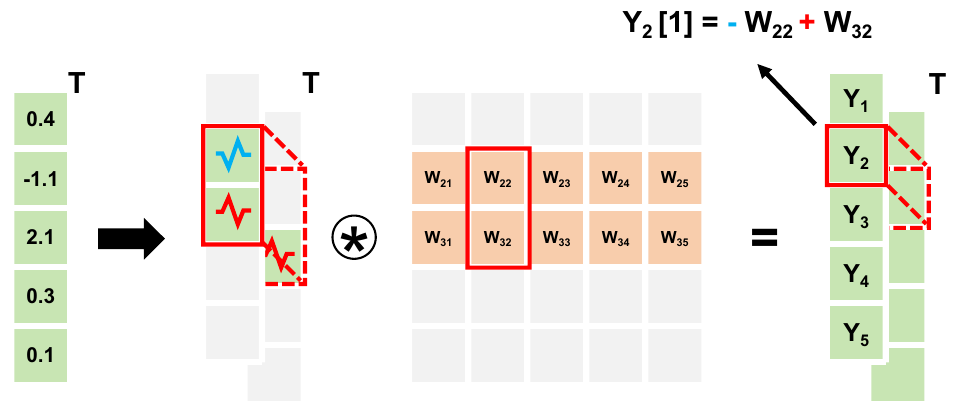}
    \caption{
    Replacing dense matrix multiplication with sparse addition via spike encoding.
    }
    \label{fig:s2}
\end{figure}

\paragraph{Step Two: From Integer Spike to 0/1 Spike.} Spike counts in integer form are converted to traditional 0/1 spike values by unfolding the virtual time step from \(T\) to \(T \times D\) \citep{luo2024integer}. Owing to the $T \times D$ virtual temporal expansion, the spike representation admits both synchronous (multi-step latency) and asynchronous continuous (single-step latency) firing modes \citep{yao2025scaling}.
Specifically, the input \( \mathbf{S}^\ell[t]\) is unfolded into a spike train \(\{\mathbf{S}^\ell[t, d] \}_{d}^{D}\), effectively converting integer values into traditional spike values, performing computations without matrix multiplication. The corresponding equations are given as follows:
\begin{align}
    \mathbf{U}^{\ell+1}[t] &= \mathbf{H}^{\ell+1}[t-1]+\sum_{d}^{D} \left(\mathbf{W}^{\ell+1} \mathbf{S}^\ell[t, d]\right).
\end{align}
Since \( \mathbf{W}^{\ell+1} \sum_{d}^{D} \mathbf{S}^\ell[t,d] = \sum_{d}^{D} \left( \mathbf{W}^{\ell+1} \mathbf{S}^\ell[t,d] \right) \), where \( \mathbf{W}^{\ell+1} \) is the corresponding weight matrix, the spike \( \mathbf{S}^\ell[t,d] \) 
can thus replace matrix multiplication with sparse addition, as formalized
in \autoref{thm:sparse_addition}.

\begin{theorem}
\label{thm:sparse_addition}
(Replacing Matrix Multiplication with Sparse Addition via Binary Spikes) Given an input spike train \( \mathbf{X} \in \{0,1\}^n \), the dense matrix multiplication \( \mathbf{Y} = \mathbf{W} \mathbf{X} \), where \( \mathbf{W} \in \mathbb{R}^{m \times n} \), is equivalent to a sparse addition over selected columns of \( \mathbf{W} \):
\begin{align}
\mathbf{Y} = \sum_{i \in \mathcal{I}} {W}_{:,i}; \quad \mathcal{I} = \{ i \mid X_i = 1 \}.
\end{align}
\end{theorem}

\begin{proof}
Since each element of the input vector \( \mathbf{X} \) is binary (\( X_i \in \{0,1\} \)), the multiplication \( W_{j,i} \cdot X_i \) simplifies to:
\begin{align}
W_{j,i} \cdot X_i =
\begin{cases}
W_{j,i}, & \text{if } X_i = 1 \\
0, & \text{if } X_i = 0
\end{cases}
\end{align}
Therefore, the matrix-vector product \( \mathbf{Y} = \mathbf{W} \mathbf{X} \) can be rewritten as a summation over the columns of \( \mathbf{W} \) corresponding to indices \( i \) where \( X_i = 1 \). This eliminates all multiplications with 0, resulting in sparse addition:
\begin{align}
\mathbf{Y} = \sum_{i \in \mathcal{I}} {W}_{:,i}.
\end{align}
This shows that when \( \mathbf{X} \) is a 0/1 spike vector, the dense matrix multiplication degenerates into a sparse spike-driven process, where only active spikes contribute to the output.
\end{proof}

We focus on matrix multiplication operators, as the majority of computation in LLMs is concentrated in matrix multiplications, while other operators, including bias and normalization, typically contribute several orders of magnitude less computational cost. For compatibility with other operators, RMSNorm and nonlinear activations can also be approximately implemented using addition operations \citep{abreu2025neuromorphic,arora2024simple}.

\subsection{Analysis and Challenges of Spike Firing}
\label{sec:4.2}
The spike-driven mechanism makes firing rate a key factor in SNN cost.
In spike-based quantization, floating-point activations are first quantized into integer spike counts and then unfolded into 0/1 spike trains with specific firing rates, thereby introducing inherent sparsity during spike unfolding.
We further study firing-rate regularities under the count–spike two-step encoding paradigm.

\paragraph{Spike count statistics.}
We first analyze the correspondence between the integer spike count $s^\ell[t]$ obtained in the first step and the spike train $s^\ell[t,d]$ unfolded in the second step of the count–spike two-step encoding, in order to characterize the sparsity induced by expanding different spike counts into spike trains. Since the spike count is equivalent to the total number of spikes in the corresponding spike train over the $D$ unfolded time steps, we define its statistic as:
\begin{align}
k^\ell[t] &= \sum_{d}^{D} \mathbf{S}^\ell[t,d]. \label{eq:8}
\end{align}

\( k^\ell[t] \) denotes the total number of spike fired by the \(\ell\)-th layer neuron over the time window. \( \mathbf{S}^\ell[t, d] \) is the spike state at time step \( t \) and virtual step \( d \) (1 if a spike is fired, 0 otherwise). The total count \( k^\ell[t] \) is obtained by summing \( \mathbf{S}^\ell[t, d] \) over all \( d \).

\paragraph{Calculate Spike Firing Rate.} 
As illustrated in the left panel of Fig.~\ref{fig:6}, the spike count obtained from Eq.~\ref{eq:8} takes discrete values 
$k^\ell \in \{0,1,\ldots,D\}$. 
The probability of each spike count value is determined by the area of the corresponding interval under the membrane potential probability density function of the $\ell$-th layer. 
Specifically, the membrane potential probability density function is partitioned into $D$ intervals, where the area of each interval corresponds to the probability that the discrete spike count takes the value $k$ at layer $\ell$, denoted as $P_k^\ell$. 
Accordingly, the firing rate can be expressed as:
\begin{align}
R^\ell &= \sum_{k} \frac{k^\ell}{D} \cdot P_k^\ell. \label{eq:9}
\end{align}

In this formula, \( R^\ell \) denotes the firing rate of the \(\ell\)-th layer neuron, where \( k^\ell \) is the integer spike count value, \( D \) is the time window length, and \( P_k^\ell \) is the probability of quantizing to integer \( k \). The firing rate is obtained by a weighted sum over all integer spike count values and their corresponding probabilities.

We visualize the spike count distributions across different layers of LLaMA2-7B in Fig.~\ref{fig:3}.
In the first step, a 4-4 bit setting is adopted (4-bit weights and 4-bit spike counts, with $T=1$),
while the second step uses a 4-1 bit setting (4-bit weights and binary spikes with $D=15$).
Taking the QKV layer as an example, the average spike count is $7.52$, with a firing rate of approximately $0.5$,
which indicates notable spike redundancy under the current encoding configuration.

\subsection{More Sparsity via Bidirectional Encoding under Symmetric Quantization} \label{sec:4.3}

To reduce the spike firing rate, we begin by analyzing the inherent sparsity patterns in the spike quantization process and the relationship between membrane potential and spike firing probability. 

\begin{theorem}
(The Relationship Between \( R^\ell \) and \( P^\ell_k \)) 
To reduce the spike firing rate \( R^\ell \), smaller integer spike count values \( k \) should correspond to larger probabilities \( P^\ell_k \).
\end{theorem}

\begin{proof}
From Eq. \ref{eq:9}, the spike firing rate \( R^\ell \) is given by:
\(
R^\ell = \sum_{k} \frac{k^\ell}{D} \cdot P^\ell_k.
\)
Let \( k^\ell_1 < k^\ell_2 < \dots < k^\ell_n \) denote the integer spike count values, with corresponding quantization probabilities \( P^\ell_{k_1} > P^\ell_{k_2} > \dots > P^\ell_{k_n} \). When \( k^\ell \) decreases, \( P^\ell_k \) increases. Since \( k^\ell_1 \) is the smallest spike count value, its corresponding probability \( P^\ell_{k_1} \) is the largest, and its contribution to the overall firing rate is dominant.
To minimize the firing rate \( R^\ell \), smaller \( k^\ell \) should be paired with larger \( P^\ell_k \). This allocation minimizes \( R^\ell \), producing a sparser binary spike train within a unit time window.
\end{proof}

\begin{figure*}
\centering
\begin{minipage}[b]{0.254\linewidth}
  \centering
  \includegraphics[width=\linewidth,height=0.156\textheight]{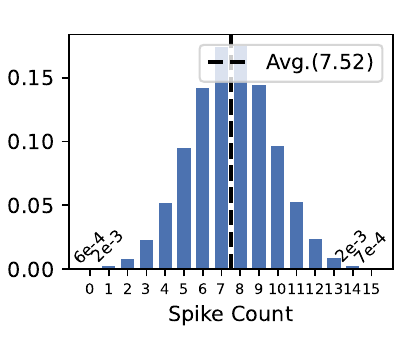}\textit{}
  \scriptsize (a) QKV (ASymmetric)
\end{minipage}
\begin{minipage}[b]{0.226\linewidth}
  \centering
  \includegraphics[width=\linewidth,height=0.156\textheight]{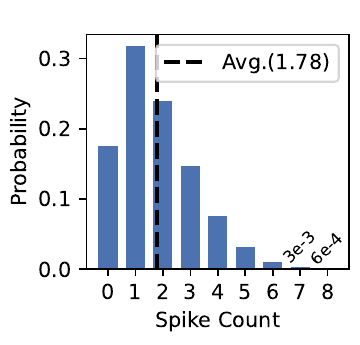}
  \scriptsize (b) QKV (Symmetric)
\end{minipage}
\begin{minipage}[b]{0.254\linewidth}
  \centering
  \includegraphics[width=\linewidth,height=0.156\textheight]{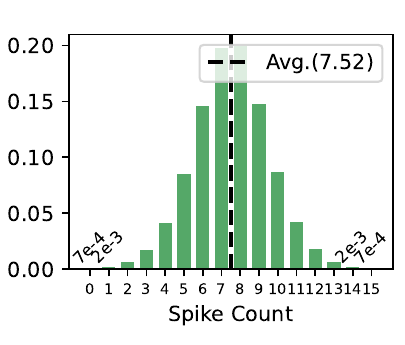}
  \scriptsize (c) OProj (ASymmetric)
\end{minipage}
\begin{minipage}[b]{0.226\linewidth}
  \centering
  \includegraphics[width=\linewidth,height=0.156\textheight]{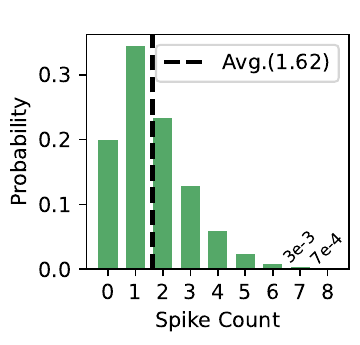}
  \scriptsize (d) OProj (Symmetric)
\end{minipage}
   \caption{Bidirectional encoding under symmetric quantization significantly reduces spike count (e.g., $7.52\rightarrow1.78$) while halving unfolded time steps (e.g., $15\rightarrow8$), dramatically lowering firing rate (e.g., $0.5\rightarrow0.2$). More results can be found in Fig.~\ref{fig:sup_d}.}
\label{fig:3}
\end{figure*}
\begin{figure*}
\centering
\begin{minipage}[t]{0.254\linewidth}
  \centering
  \includegraphics[width=\linewidth,height=0.156\textheight]{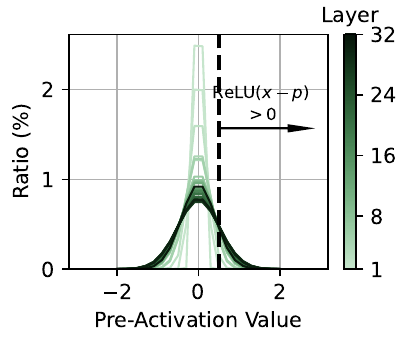}
  \scriptsize (a) QKV ($q$=0.5)
\end{minipage}\hfill
\begin{minipage}[t]{0.226\linewidth}
  \centering
  \includegraphics[width=\linewidth,height=0.156\textheight]{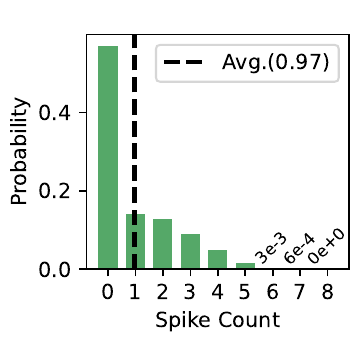}
  \scriptsize (b) QKV ($q$=0.5)
\end{minipage}
\hspace{0.15pt} 
\begin{minipage}[t]{0.254\linewidth}
  \centering
  \includegraphics[width=\linewidth,height=0.156\textheight]{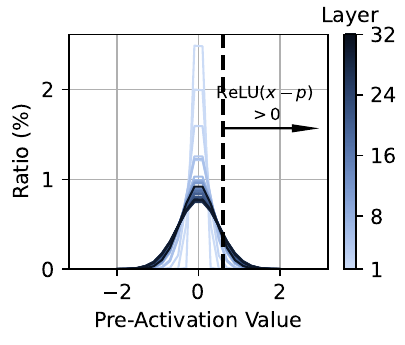}
  \scriptsize (c) QKV ($q$=0.6)
\end{minipage}\hfill
\begin{minipage}[t]{0.226\linewidth}
  \centering
  \includegraphics[width=\linewidth,height=0.156\textheight]{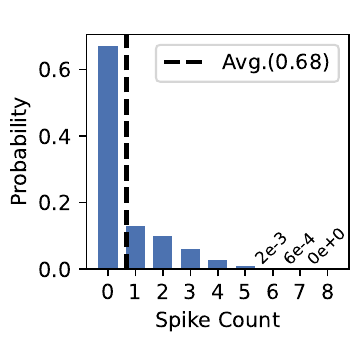}
  \scriptsize (d) QKV ($q$=0.6)
\end{minipage}
    \caption{Spike count is further reduced by membrane potential clipping via quantile-based ReLU (e.g., $7.52\rightarrow0.97$), achieving lower firing rate (e.g., $0.5\rightarrow0.1$).}
    \label{fig:4}
\end{figure*}

As shown in Fig.~\ref{fig:6} (left), asymmetric quantization tends to concentrate the probability of spike count values k in the middle range, with relatively low probability at the extremes. To alleviate this issue, we introduce a bidirectional spike encoding strategy based on symmetric quantization within the $\gamma$-SQP two-step spiking framework. This strategy improves firing sparsity in spike-based large language models while halving the temporal expansion. Specifically, membrane potentials are first symmetrically quantized by $\gamma$-SQP into polarity integer counts, which are then expanded into ternary spike trains (-1/0/1) through bidirectional temporal encoding.

Ternary spikes are a generalized extension of binary spikes, adding polarity information while preserving spike-driven computation \citep{xing2023spikelm,guo2024ternary}. Our focus 
is not on the spike representation itself, as binary or ternary spikes themselves do not determine the sparsity of the spike trains. Instead, our approach focuses on achieving sparser spike trains through our $\gamma$-SQP two-step spike encoding method, combined with symmetric quantization and bidirectional encoding. Ternary spikes, due to their polarity, are well-suited as the output representation for our method.

\paragraph{Step One.} 
To mitigate the high firing rate caused by asymmetrically quantized spikes, we adopt bidirectional spike encoding under symmetric quantization. Specifically, polarity spike counts are first obtained via our symmetric quantization with $\gamma$-SQP. Accordingly, Eq.~\ref{eq:GMQuaRot_LIF} is rewritten as:
\begin{align}
    \mathbf{S}^\ell[t]&=\gamma\text{-SQP}\left(\mathbf{U}^\ell[t]),-\frac{D}{2}-1, \frac{D}{2}\right).
\end{align}

Similarly, we adopt the $\gamma$-SQP method introduced in Section~\ref{sec:4.1} to incorporate a Hadamard rotation, but apply symmetric quantization to the rotated membrane potential, i.e., $\mathbf{U} \leftarrow \mathbf{U}\,\mathrm{Diag}(\gamma)\mathbf H$. Specifically, the rotated membrane potential is quantized to integer values and constrained to the interval $\left[-\frac{D}{2}-1,\; \frac{D}{2}\right]$, yielding the integer-valued spike count $\mathbf{S}^\ell[t]$.

\paragraph{Step Two.}
To generate bidirectionally encoded spikes, we map integer spike counts by unfolding the virtual time steps from $T$ to
$
T \times \max\!\left(\left|-\frac{D}{2}-1\right|,\; \left|\frac{D}{2}\right|\right).
$
The integer spike count $\mathbf{S}^\ell[t]$ is unfolded into a spike train
$\{\mathbf{S}^\ell[t,d]\}_{d=1}^{\frac{D}{2}+1}$, distributing the original integer spike count into a temporally expanded bidirectional spike train. Owing to the increased representational capacity introduced by bipolar spikes, the required temporal length is reduced from $D$ to $\frac{D}{2}+1$. The corresponding computation is given by
\begin{align}
\mathbf{U}^{\ell+1}[t]
=
\mathbf{H}^{\ell+1}[t-1]
+
\sum_{d=1}^{\frac{D}{2}+1}
\bigl(
\mathbf{W}^{\ell+1} \mathbf{S}^\ell[t,d]
\bigr).
\end{align}

On one hand, because bidirectional spikes carry polarity, they can represent twice the amount of information within a single time step compared to non-polarized spikes, thus halving the required number of time steps. On the other hand, due to the two-step symmetric quantization encoding method, the corresponding counting scheme for specific membrane potential mapping regions has changed.  The original spike train with a count of 7, \( \{1, 1, 1, 1, 1, 1, 1, 0, \dots, 0\}_{15} \), is transformed into a spike train with a count of -1, \( \{-1, 0, \dots, 0\}_8 \), after bidirectional encoding under symmetric quantization. Bidirectional spike encoding simultaneously preserves the spike-driven nature and improves efficiency, as illustrated in Fig. \ref{fig:s2}. 

\begin{theorem}
(Replacing Matrix Multiplication with Sparse 
Addition via Ternary Spikes) Given an input vector \( \mathbf{X} \in \{-1, 0, 1\}^n \), the matrix multiplication \( \mathbf{Y} = \mathbf{W} \mathbf{X} \) can be equivalently computed as a sparse accumulation over selected columns of \( \mathbf{W} \):
\begin{align}
&\mathbf{Y} = \sum_{i \in \mathcal{I}_+} {W}_{:,i} + \sum_{i \in \mathcal{I}_-} (-{W}_{:,i}); \\
&\mathcal{I}_+ = \{i \mid X_i = 1\},\ \mathcal{I}_- = \{i \mid X_i = -1\}.
\end{align}
\end{theorem}

\begin{proof}
Each nonzero element in \( \mathbf{X} \) represents an event-triggered spike at index \( i \), and contributes to the output according to:
\begin{align}
\tilde{{W}}_{:,i} =
\begin{cases}
\begin{array}{@{}rl}
{W}_{:,i}, & \text{if } X_i = 1 \\
- {W}_{:,i}, & \text{if } X_i = -1 \\
0, & \text{if } X_i = 0
\end{array}
\end{cases}
\end{align}

Thus, instead of computing \( \mathbf{W} \mathbf{X} \) through dense multiply-accumulate, we perform sparse selection and signed accumulation over active spike positions:
\begin{align}
\mathbf{Y} = \sum_{i \in \mathcal{I}_+} {W}_{:,i} + \sum_{i \in \mathcal{I}_-} (-{W}_{:,i}).
\end{align}

This sparse formulation eliminates multiplications and reflects the spike-driven nature of Ternary spikes, where each spike determines whether and how the corresponding column contributes to the final output.
\end{proof}

\paragraph{Spike firing count.}
Under the bidirectional spike encoding scheme, neuronal activations are represented by spike trains with different polarities. Although spikes may have different polarities, all nonzero spike events are considered functionally equivalent activations. Accordingly, the total spike activity at layer $\ell$ is defined as
\begin{align}
k^\ell = \sum_{t=1}^{T} \sum_{d=1}^{\frac{D}{2}+1} \left| \mathbf S^\ell[t,d] \right|.
\end{align}
Here, $k^\ell$ denotes the aggregate number of spikes counted irrespective of polarity, providing a unified measure of overall activation strength along the temporal dimension under bidirectional encoding.

As shown in the middle panel of Fig.~\ref{fig:6}, in the symmetric spike quantization mode, the mapping range of the membrane potential narrows as the number of spikes within the unit time window increases, {causing the mapping probability to decrease,} resulting in a significantly lower overall firing rate compared to the asymmetric mode. In Fig.~\ref{fig:3}, we present the results on the QKV layer of LLaMA, where the average spike count is reduced from 7.52 to 1.78 and the firing rate decreases from 0.5 to 0.22 under symmetric quantization, and the unfolded time steps are simultaneously reduced from 15 to 8.

\subsection{More Sparsity Achieved via Membrane Potential Clipping} \label{sec:4.4}
In addition to symmetric spike quantization, we further explore enhancing sparsity by modifying the initialization of the membrane potential distribution through clipping. As shown in Fig.~\ref{fig:6} (right), the majority of the membrane potential distribution is mapped to the 0-valued spike count, while only a small clipped portion of the distribution is progressively mapped to spike counts from 1 to the maximum value. This design significantly increases the proportion of the probability mapping area corresponding to spike trains expanded from 0-valued spike counts, thereby further reducing the overall spike firing rate (see Fig.~\ref{fig:4}).

\paragraph{Quantile-Smoothed ReLU.}
The ReLU (Rectified Linear Unit) activation has been widely shown to promote sparsity in conventional models. 
Inspired by this property, with reference to the method proposed in ReLU Strikes Back \citep{mirzadehrelu}, we propose a variant more suitable for our setting, termed Quantile-Smoothed ReLU.
This method is applied as a lightweight fine-tuning step on the baseline prior to the two-step encoding with $\gamma$-SQP, where a quantile-based threshold is introduced to partition the membrane potential and enhance activation sparsity.
Specifically, during fine-tuning, we maintain and record a moving average of the quantile of the membrane potential distribution; during inference, the recorded quantile threshold is directly reused without recomputing the quantile, thereby avoiding additional computational overhead. The Quantile-Smoothed ReLU is defined as:

\begin{align}
\mathbf U_{\mathrm{sp}}^\ell(t)
&= \mathrm{ReLU}\!\left(\mathbf U^\ell(t) - \tau^\ell(t)\right); \label{eq:relu_memb}\\
\tau^\ell(t)
&= \mathrm{Quantile}\!\left(\mathbf U^\ell(t), q\right).
\end{align}

Here, $\mathbf U^\ell(t)$ denotes the membrane potential of the neuron at layer $\ell$ and time step $t$, and $q$ is the quantile ratio factor. 
$\tau^\ell(t)$ is used to determine the membrane-potential threshold corresponding to the $q$-th quantile, thereby partitioning the membrane potential. 
This quantile is updated via a moving average:
\begin{align}
\begin{aligned}
\tau^\ell(t)
&= \alpha\,\tau^\ell(t-1) + (1-\alpha)\,{\tau}^\ell(t),
\end{aligned}
\end{align}
where $\alpha$ is a moving-average factor. 
It is worth noting that the computation and update of the quantile are performed only during the fine-tuning stage.
The resulting output $\mathbf U_{\mathrm{sp}}^\ell(t)$ represents the membrane potential after being partitioned by the Quantile-Smoothed ReLU. 
Subsequently, we rewrite Eq.\ref{eq:GMQuaRot_LIF} as:
 \begin{align}
 \mathbf S^\ell[t] &= \gamma\text{-SQP}\left(\mathbf{U}^\ell_{\text{sp}}(t)),0,D\right).
\end{align}

\paragraph{Joint Sparsity and Rotation Matrices.}
We further explore the joint effect of rotation matrices and sparsification. By learning sparsity from computational invariance, we construct
$
\mathrm{ReLU}(\mathbf{XQ})\,\mathbf{Q}^{\mathsf{T}}\mathbf{W},
$
to simultaneously improve spike quantization performance and sparsity . Based on this construction, Eq.~\ref{eq:relu_memb} can be further rewritten as:
\begin{align}
\mathbf{U}_{\mathrm{sp}}^\ell(t)
=
\mathrm{ReLU}\!\left(
\mathbf{U}^\ell(t)\mathbf{Q}
-
\tau^\ell(t)
\right).
\end{align}
For the next layer, we have
$
f\!\left(\mathbf{W}^{\ell+1}\mathbf{Q},\, \mathbf{S}^\ell[t]\right).
$
The membrane potential is first transformed by the rotation matrix and then processed by the quantile-smoothed ReLU to obtain sparse membrane potentials; subsequently, these spike signals are combined with the rotated weights via sparse addition.

\begin{table*}
\centering
\setcaptionwidth{1} 
\caption{Zero-shot QA (↑) results between SDLLM and SpikeLLM under SpikeLLM settings. 
'Bit' denotes the bit-width of the weights and activations, respectively. '$R$' denotes the firing rate. 
The $T \times D$ time-step setting with an unfolded dimension $D$ supports synchronous (multi-step latency) and asynchronous (single-step latency) spike firing.
 FLOPs (↓) are INT FLOPs, the same as \citet{xu2023qdetr,liu2020birealnet,xu2025neurormorphic}. Power (↓) is the estimated energy consumption, the same as \citet{yao2025scaling,qiu2025quantized,yao2024spikev2}, { further details are provided in Appendix \ref{app:details_of_op_ec}.}}
\label{tab:1}
\setlength{\tabcolsep}{3pt}
\vspace{0.5em}
\begin{adjustbox} {width=\linewidth} 
\begin{tabular}{l c c c c c c c c c c c c@{\hspace{3pt}}c}
\toprule
Model & Spike & Bit & $T \times D$ & $R$ & PIQA & ARC-e & ARC-c & BoolQ & HellaS & WinoG & Avg. & FLOPs(T)  & Power(J)\\
\midrule
LLaMA2-7B    & \ding{55} & -   & - & -          & 78.84 & 74.54 & 46.33 & 77.74 & 75.97 & 69.22 & 70.44 & 6.91 & 31.77 
\\
QuaRot             & \ding{55} & 4-4  & - & - & 71.82 & 59.89 & 36.18 & 67.37 & 63.88 & 59.12 & 59.71 & 0.86 & 3.97  
\\
SpikeLLM           & \ding{51} & 4-4  & $1.2 \times 1$ & - & 72.47 & 62.29 & 36.01 & 69.48 & 64.74 & 59.43 & 60.74 & 1.04 & 4.77 
\\
\rowcolor{lightpink}
SDLLM              & \ding{51} & 4-1.58  & $1 \times 8$ & 0.216 & 75.84 & 69.65 & 41.21 & 74.01 & 71.75 & 66.14 & \textbf{66.43} & \textbf{0.75} &\textbf{0.67}
\\ 	 
\midrule
LLaMA2-13B   &\ding{55}  & -   & -  & - & 80.63 & 77.48 & 49.23 & 80.73 & 79.37 & 71.74 & 80.69 & 13.42 & 61.74  
\\
QuaRot             &\ding{55}  & 4-4  & - & - & 74.86 & 69.19 & 41.98 & 72.54 & 70.35 & 64.72 & 65.61 & 1.68 & 7.72 
\\
SpikeLLM           & \ding{51} & 4-4  & $1.2 \times 1$ & - & 75.79 & 69.53 & 41.21 & 74.31 & 71.51 & 65.51 & 66.31 & 2.01 & 9.26 
\\
\rowcolor{lightpink}
SDLLM              & \ding{51} & 4-1.58  & $1 \times 8$ & 0.209 & 78.51 & 74.12 &46.16 & 78.26 & 76.36 & 69.85 & \textbf{70.54} & \textbf{1.40} &\textbf{1.26}\\
\bottomrule
\end{tabular}
\end{adjustbox}

\vspace{4pt}

\centering
\caption{Zero-shot QA (↑) with Membrane Potential Clipping: Lower Firing Rate Enhances Efficiency.}
\label{tab:2}
\setlength{\tabcolsep}{3pt}
\vspace{0.5em}
\begin{adjustbox}{width=\linewidth} 
\begin{tabular}{lcccccccccccccc}
\toprule
\multirow{2}{*}{Model} &\multirow{2}{*}{Bit} & \multicolumn{5}{c}{QKV} & \multirow{2}{*}{PIQA}& \multirow{2}{*}{ARC-E} &\multirow{2}{*}{ARC-C} &\multirow{2}{*}{BoolQ} &\multirow{2}{*}{HellaS} &\multirow{2}{*}{WinoG}&\multirow{2}{*}{Avg.}\\
\cmidrule(lr){3-7} 
& & $\textit{q}$ & $T \times D$ & $R$ & FLOPs(T) & Power(J)  \\
\midrule
LLaMA2-7B & - & -   & -                      & -              & 1.649 & 7.585 & 78.84 & 74.54 & 46.33 & 77.74 & 75.97 & 69.22 & 70.44
\\
QuaRot & 4-4     & -   & -                      & -  & 0.206 & 0.948 & 71.82 & 59.89 & 36.18 & 67.37 & 63.88 & 59.12 & 59.71 
\\
SpikeLLM & 4-4   & -   & 1.2 $\times$ 1                    & -  & 0.247 & 1.136 & 72.47 & 62.29 & 36.01 & 69.48 & 64.74 & 59.43 & 60.74
\\
\rowcolor{lightpink}
SDLLM & 4-1.58    & -   & 1 $\times$ 8 & 0.216 & \textbf{0.184} & \textbf{0.166} & 75.84 & 69.65 & 41.21 & 74.01 & 71.75 & 66.14 & \textbf{66.43}
\\
\rowcolor{lightpink}
SDLLM$_q$ & 4-1.58    & 0.5 &  1 $\times$ 8  & 0.120  & \textbf{0.100} & \textbf{0.090} & 73.94 & 59.22 & 34.30 & 71.71 & 64.30 & 63.61 & \textbf{61.18} 
 \\
\rowcolor{lightpink}
SDLLM$_q$ & 4-1.58     & 0.6 & 1 $\times$ 8  & 0.100  & \textbf{0.083} & \textbf{0.075} & 73.07 & 61.15 & 34.13 & 69.60 & 63.57 & 60.85 & 60.40 
\\
\bottomrule
\end{tabular}
\end{adjustbox}
\end{table*}

\begin{table}
\centering
\setcaptionwidth{1} 
\caption{Comparison of PPL (↓) metrics on Wikitext2 and C4 for LLaMA2-7B and 13B between SDLLM and QuaRot.}
\label{tab:app_llama_ppl}
\large \fontseries{sb}\selectfont
\setlength{\tabcolsep}{3pt}
\vspace{0.5em}
\begin{adjustbox}{width=1\linewidth} 
\begin{tabular}{@{}lcccccc@{\hspace{2pt}}c@{}}
\toprule

Model  & Bit &  $T \times D$ & $R$ & Wiki & C4  & FLOPs(T) & Power(J) \\
\midrule
LLaMA2-7B  & - & - & - & 5.47 & 7.26  & 6.91 & 31.77  \\
SmoothQuant    & 4-4 & - & - & 83.12 & 77.27 & 0.86 & 3.97  \\
OmniQuant      & 4-4 & - & - & 14.26 & 18.02 & 0.86 & 3.97  \\
AfineQuant     & 4-4 & - & - & 12.69 & 15.76 & 0.86 & 3.97  \\
QLLM            & 4-4 & - & - & 11.45 & 13.26  & 0.86 & 3.97  \\
Atom          & 4-4 & - & - & 8.40  & 10.96  & 0.86 & 3.97  \\
QuaRot      & 4-4 & - & - & 8.73  & 12.27 & 0.86 & 3.97  \\
\rowcolor{lightpink}
SDLLM     & 4-1.58 & 1 $\times$ 8 & 0.216 & \textbf{6.40} & \textbf{8.58}  & \textbf{0.75} & \textbf{0.67} \\
\rowcolor{lightpink}
SDLLM         & 4-1.58 & 1 $\times$ 16 & 0.221 & \textbf{5.95} & \textbf{7.93} & 1.53 & \textbf{1.37} \\
\midrule
LLaMA2-13B  & - & - & - & 4.88 & 6.73  & 13.42 & 61.74  \\
SmoothQuant     & 4-4 & - & - & 35.88 & 43.19  & 1.68 & 7.72  \\
OmniQuant      & 4-4 & - & - & 12.30 & 14.55  & 1.68 & 7.72  \\
AfineQuant     & 4-4 & - & - & 11.75 & 13.97 & 1.68 & 7.72  \\
QLLM           & 4-4 & - & - & 9.09  & 11.13  & 1.68 & 7.72  \\
Atom             & 4-4 & - & - & 6.96  & 9.12   & 1.68 & 7.72  \\
QuaRot      & 4-4 & - & - & 6.31  & 9.02   & 1.68 & 7.72  \\
\rowcolor{lightpink}
SDLLM        & 4-1.58 & 1 $\times$ 8 & 0.209 & \textbf{5.48} & \textbf{7.59}  & \textbf{1.40} & \textbf{1.26} \\
\rowcolor{lightpink}
SDLLM       & 4-1.58 & 1 $\times$ 16 & 0.215 & \textbf{5.18} & \textbf{7.15}  & 2.89 & \textbf{2.60} \\
\bottomrule
\end{tabular}
\end{adjustbox}
\end{table}

\section{Experiments}

We apply the SDLLM method (our $\gamma$-SQP two-step spike encoding with symmetric quantization–based bidirectional encoding and segmented membrane potentials as sparse components, which replaces dense matrix multiplications with spike-driven sparse additions) to LLaMA2-7B, LLaMA2-13B, and LLaMA3-8B, as well as the newer LLM Qwen2.5-14B, and systematically evaluate performance on commonsense question answering (PIQA, ARC-easy, ARC-challenge, HellaSwag, and WinoGrande) and more complex language generation tasks, including reading comprehension (BoolQ and SQuAD), world knowledge (TriviaQA), and 
mathematical reasoning (GSM8K). Our core focus is to explore whether LLMs at the tens-of-billions scale can be constructed as spike-driven SNNs that replace dense matrix multiplications with sparse additions through spike encoding.
This work provides an initial exploration of this direction and offers new insights into leveraging spike sparsity in LLM settings, while also providing inspiration for the broader neuromorphic spiking computing ecosystem.

\begin{table*}
\renewcommand{\arraystretch}{0.88}
\centering
\setcaptionwidth{1} 
\caption{Evaluation of Zero-shot QA (↑) results of LLaMA2-7B and 13B under QLLM settings.}
\label{tab:3}
\setlength{\tabcolsep}{3pt}
\vspace{0.5em}
\begin{adjustbox}{width=\linewidth} 
\begin{tabular}{l c c c c c c c c c c c c@{\hspace{3pt}}c}
\toprule
Model & Spike & Bit & $T \times D$ & $R$ & PIQA & ARC-e & ARC-c & BoolQ & HellaS & WinoG & Avg. & FLOPs(T)  & Power(J)\\
\midrule
LLaMA2-7B  & \ding{55} & - & - & -          & 76.88 & 53.54 & 40.53 & 71.13 & 72.96 & 67.25 & 63.72 &  6.91 & 31.77  \\
SmoothQuant      & \ding{55} & 4-4 & - & - & 60.17 & 35.23 & 27.13 & 57.92 & 37.08 & 49.57 & 44.52 & 0.86 & 3.97   \\ 
OS+              & \ding{55} & 4-4 & - & - & 63.11 & 39.10 & 28.84 & -     & 51.30 & 45.93 & 45.66 &  0.86 & 3.97   \\ 
OmniQuant        & \ding{55} & 4-4 & - & - & 65.61 & 44.28 & 30.38 & 62.66 & 53.51 & 51.85 & 51.38 & 0.86 & 3.97   \\
AffineQuant      & \ding{55} & 4-4 & - & - & 67.36 & 44.23 & 31.91 & 62.75 & 54.34 & 55.18 & 52.64 & 0.86 & 3.97   \\ 
QLLM             & \ding{55} & 4-4 & - & - & 67.68 & 45.29 & 32.09 & 62.42 & 58.45 & 56.59 & 51.60 &  0.86 & 3.97   \\ 
Atom             & \ding{55} & 4-4 & - & - & 69.75 & 47.35 & 34.22 & 62.42 & 63.21 & 56.51 & 55.58 & 0.86 & 3.97   \\ 
DuQuant          & \ding{55} & 4-4 & - & - & 75.24 & 51.89 & 36.77 & 67.86 & 69.54 & 62.12 & 60.57 &  0.86 & 3.97   \\ 
\rowcolor{lightpink}
SDLLM            & \ding{51} & 4-1.58 & 1 $\times$ 8 & 0.216 &74.54 	&51.89 	&38.74 	&68.81 	&69.00 	&63.54 	&\textbf{61.09} & \textbf{0.75}& \textbf{0.67} \\    
\midrule
LLaMA2-13B & \ding{55}  & - & - & -          & 79.05 & 57.91 & 44.20 & 69.02 & 76.60 & 69.69 & 66.08 & 13.42 & 61.74  \\
SmoothQuant      & \ding{55} & 4-4 & - & - & 62.30 & 40.28 & 30.72 & 60.49 & 42.24 & 49.96 & 47.67 &  1.68  & 7.72  \\
OS+              & \ding{55} & 4-4 & - & - & 64.47 & 41.46 & 32.17 & -     & 59.30 & 51.38 & 49.76 &  1.68  & 7.72  \\
OmniQuant        & \ding{55} & 4-4 & - & - & 69.80 & 47.22 & 33.79 & 65.47 & 59.34 & 55.49 & 55.19 &  1.68  & 7.72  \\
AffineQuant      & \ding{55} & 4-4 & - & - & 68.55 & 47.64 & 32.34 & 66.97 & 59.97 & 55.07 & 55.09 &  1.68  & 7.72  \\
QLLM             & \ding{55} & 4-4 & - & - & 70.46 & 48.48 & 34.39 & -     & 62.80 & 55.41 & 54.31 &  1.68  & 7.72  \\
Atom             & \ding{55} & 4-4 & - & - & 71.16 & 50.89 & 37.88 & 63.91 & 67.51 & 58.40 & 58.29 &  1.68  & 7.72  \\
DuQuant          & \ding{55} & 4-4 & - & - & 77.31 & 55.60 & 41.55 & 66.61 & 73.68 & 66.06 & 63.47 &  1.68  & 7.72  \\
\rowcolor{lightpink}
SDLLM            & \ding{51} & 4-1.58 & 1 $\times$ 8 & 0.209 & 77.26 &	57.41 &	41.55 &	66.67 &	73.33 &	66.69  &\textbf{63.82} &  \textbf{1.40} 	&\textbf{1.26}
\\
\bottomrule
\end{tabular}
\end{adjustbox}
\vspace{4pt}
\renewcommand{\arraystretch}{0.88}
\centering
\setcaptionwidth{1} 
\caption{Evaluation of Zero-shot QA (↑) results of Qwen2.5-14B.} 
\label{tab:app_zero_shot_qwen}
\setlength{\tabcolsep}{3pt}
\vspace{0.5em}
\begin{adjustbox}{width=\linewidth} 
\begin{tabular}{l c c c c c c c c c c c c@{\hspace{3pt}}c}
\toprule
Model & Spike & Bit & $T \times D$ & $R$ & PIQA & ARC-e & ARC-c & BoolQ & HellaS & WinoG & Avg. & FLOPs(T)  & Power(J)\\
\midrule
Qwen2.5-14B  & \ding{55} & - & - & - & 82.10 & 79.12 & 58.87 & 85.26 & 82.91 & 75.30 & 77.26  &13.53 &	62.23
\\
RTN         & \ding{55} & 4-4 & - & -  & 51.31 & 32.91 & 24.32 & 50.40 & 29.29 & 47.91  & 39.35  & 1.69 & 7.78 \\
GPTQ        & \ding{55} & 4-4 & - & - & 51.80 & 26.64 & 23.63 & 41.13 & 26.27 & 49.17 & 36.73 & 1.69 & 7.78 \\ 
SmoothQuant & \ding{55} & 4-4 & - & -  & 51.20 & 26.09 & 26.54 & 41.13 & 26.27 & 49.17 & 36.73  & 1.69 & 7.78 \\ 
\rowcolor{lightpink}
SDLLM      & \ding{51} & 4-1.58 & 1 $\times$ 8 & 0.212 & 79.00 & 77.82 & 52.13 & 80.00 & 78.28 & 67.80 & \textbf{72.51}  & \textbf{1.43} &	 \textbf{1.29}  \\
\rowcolor{lightpink}
SDLLM & \ding{51} & 4-1.58 & 1 $\times$ 16 & 0.213 & 81.28& 80.43 & 55.29& 82.64& 81.41& 74.82 & \textbf{76.15}  & 2.88 &	 \textbf{2.59}  \\
\midrule
RTN         & \ding{55} & 6-6 & - & -  & 80.41& 81.40 & 56.57 & 84.19 & 81.52 & 71.27  & 75.89  &3.81 &	17.50 \\
GPTQ        & \ding{55} & 6-6  & - & - & 79.71 & 76.85 & 52.82 & 80.24 & 80.11 & 70.17 & 73.32  &3.81 & 17.50 \\ 
SmoothQuant & \ding{55} & 6-6 & - & -  & 79.33 & 78.96 & 55.03 & 80.89 & 79.12 & 68.67 & 73.66  &3.81 & 17.50 \\ 

\rowcolor{lightpink}
SDLLM       & \ding{51} & 6-1.58 & 1 $\times$ 32 & 0.215 & 82.48& 79.04 & 57.76& 84.83 & 82.85& 75.22 & \textbf{77.03} & 8.73 &	 \textbf{7.85}  \\
\bottomrule
\end{tabular}
\end{adjustbox}
\end{table*}

\begin{table*}
\vspace{4pt}
\setcaptionwidth{1} 
\caption{Evaluation of a broader and more challenging range of tasks (↑) on the LLaMA family: reading Comprehension (SQuAD), world Knowledge (TriviaQA), and math (GSM8K). 
Metrics include strict accuracy (Str.), flexible accuracy (Flex.), exact match (EM), token-level F1, and Has-Answer (HA) and No-Answer (NA) subsets.
}
\label{tab:comparison_methods}
\setlength{\tabcolsep}{3pt}
\vspace{0.5em}
\begin{adjustbox}{width=\linewidth} 
\begin{tabular}{lccccccccccccccccccc}
\toprule
\multirow{2}{*}{Model} & \multirow{2}{*}{Spike}  & \multirow{2}{*}{Bit} & \multirow{2}{*}{$T \times D$} &  \multirow{2}{*}{$R$} & \multicolumn{2}{c}{GSM8K} & \multicolumn{5}{c}{SQuAD} & TriviaQA  & FLOPs & Power \\
\cmidrule(lr){6-7}
\cmidrule(lr){8-12}
& &  &  &  & Str.&Flex. &EM &F1 &HA-EM &HA-F1 &NA-F1 &EM &(T) &(J)\\
\midrule
LLaMA2-7B & \ding{55} & - & - & - & 15.30 & 15.30 & 16.77 & 24.21 & 24.68 & 39.37 & 8.66 & 64.14 &  6.91 & 31.77  \\
QuaRot & \ding{55} & 4-4 & - & - & 1.97 & 2.73 & 17.04 & 24.76 & 29.01 & 44.26 & 4.77 & 33.20 &  0.86 & 3.97  \\
\rowcolor{lightpink}
SDLLM & \ding{51} & 4-1.58 & 1 $\times$ 8 & 0.216 & 6.36 & 6.67 & 19.16 & 25.52 & 18.96 & 31.52 & 19.37 & 51.71  & \textbf{0.75} & \textbf{0.67} \\
\rowcolor{lightpink}
SDLLM & \ding{51} & 4-1.58 & 1 $\times$ 16 & 0.221 & 10.00 & 10.15 & 13.88 & 22.53 & 20.76 & 37.86 & 6.82 & 58.86  & 1.53 & \textbf{1.37} \\
\midrule
QuaRot & \ding{55} & 6-6 & - & - & 13.48 & 13.79 & 16.50 & 24.37 & 22.22 & 37.77 & 10.64 & 62.72  & 1.94 & 8.94  \\
\rowcolor{lightpink}
SDLLM &\ding{51} & 6-1.58 & 1 $\times$ 32 & 0.222 & 14.85 & 1.39 & 18.05 & 25.12 & 24.82 & 38.77 & 11.12 & 64.05  & 4.60 & \textbf{4.14} \\
\midrule
LLaMA2-13B & \ding{55} & - & - & - & 22.73 & 22.88 & 22.77 & 29.67 & 37.19 & 50.82 & 7.98 & 70.45  & 13.42 & 61.74  \\
QuaRot & \ding{55} & 4-4 & - & - & 12.42 & 12.42 & 21.96 & 29.62 & 41.78 & 56.90 & 1.64 & 52.25  & 1.68 & 7.72  \\
\rowcolor{lightpink}
SDLLM & \ding{51} & 4-1.58 & 1 $\times$ 8 & 0.209 & 15.91 & 16.21 & 21.29 & 28.49 & 37.52 & 51.75 & 4.64 & 62.44  & \textbf{1.40} & \textbf{1.26} \\
\rowcolor{lightpink}
SDLLM & \ding{51} & 4-1.58 & 1 $\times$ 16 & 0.215 & 21.06 & 21.36 & 21.15 & 28.82 & 33.07 & 48.21 & 8.94 & 66.45 & 2.89 & \textbf{2.60} \\
\midrule
QuaRot & \ding{55} & 6-6 & - & - & 21.82 & 21.97 & 23.85 & 30.42 & 41.25 & 54.24 & 6.00 & 69.14  & 3.77 & 17.36  \\
\rowcolor{lightpink}
SDLLM & \ding{51} & 6-1.58 & 1 $\times$ 32 & 0.217 & 21.52 & 21.82 & 21.99 & 29.06 & 34.46 & 48.43 & 9.21 & 69.84  & 8.74 & \textbf{7.86} \\
\midrule
LLaMA3-8B & \ding{55} & - & - & - & 48.64 & 49.55 & 26.71 & 32.58 & 52.69 & 64.30 & 0.07 & 71.58  & 7.97 & 36.67  \\
\rowcolor{lightpink}
SDLLM & \ding{51} & 4-1.58 & 1 $\times$ 8 & 0.210 & 18.80 & 19.41 & 19.65 & 27.43 & 37.66 & 52.59 & 0.05 & 49.31  & \textbf{0.84} & \textbf{0.75 }\\
\rowcolor{lightpink}
SDLLM & \ding{51} & 4-1.58 & 1 $\times$ 16 & 0.211 & 31.84 & 32.60 & 24.55 & 30.89 & 48.37 & 60.89 & 0.14 & 60.40  & 1.68 & \textbf{1.51} \\
\rowcolor{lightpink}
SDLLM & \ding{51} & 6-1.58 & 1 $\times$ 32 & 0.213 & 46.55 & 47.16 & 28.32 & 34.74 & 54.31 & 66.61 & 0.05 & 70.26  & 5.09 & \textbf{4.58} \\
\bottomrule
\end{tabular}
\end{adjustbox}
\end{table*}

\begin{table}
\renewcommand{\arraystretch}{0.88}
\centering
\setcaptionwidth{1}
\caption{Evaluation of PPL (↓) metrics on Wikitext2 and C4 for Qwen2.5-14B.}
\label{tab:app_ppl_qwen}
\large \fontseries{sb}\selectfont
\setlength{\tabcolsep}{3pt}
\vspace{0.5em}
\begin{adjustbox}{width=1\linewidth} 
\begin{tabular}{@{}lcccccc@{\hspace{2pt}}c@{}}
\toprule
Model   & Bit & $T \times D$ & $R$ & Wiki & C4 & FLOPs(T)  & Power(J)\\
\midrule
Qwen2.5-14B  & - & - & - & 5.29 & 10.35 &  13.53 & 62.23 \\
RTN  & 4-4 & - & - & 2e3 & 2e3 &  1.69 & 7.78 \\
GPTQ  & 4-4 & - & - & 6e3 & 4e3 &  1.69 & 7.78 \\
SmoothQuant  & 4-4 & - & - & 2e4 & 2e4  & 1.69 & 7.78 \\
\rowcolor{lightpink}
SDLLM & 4-1.58 & 1 $\times$ 8 & 0.212 & \textbf{8.19} & \textbf{16.12} &  \textbf{1.43} & \textbf{1.29} \\
\rowcolor{lightpink}
SDLLM  & 4-1.58 & 1 $\times$ 16 & 0.213 & \textbf{6.13} & \textbf{11.14} &  2.88 & \textbf{2.59} \\
\bottomrule
\end{tabular}
\end{adjustbox}
\end{table}

\subsection{Main Results}
\paragraph{Comparisons with SpikeLLM.} 
As shown in Table~\ref{tab:1}, we compare SDLLM with SpikeLLM, both of which are built upon the QuaRot rotation-based quantization framework for weight quantization and adopt round-to-nearest (RTN) for weight encoding. Experimental results show that, on the LLaMA-2-7B and LLaMA-2-13B models, SDLLM achieves higher task performance by 5.69\% and 4.23\%, respectively, compared to SpikeLLM. Additionally, SDLLM reduces FLOPs by 1.4$\times$, and energy consumption by 7$\times$, outperforming SpikeLLM in both accuracy and efficiency. In the time step configuration, SpikeLLM uses a single timestep $T$, while SDLLM employs a $T \times D$ time step, where $D$ denotes the unfolded time step, supporting both synchronous (multi-step latency) and asynchronous (single-step latency) spike firing modes, as shown in Fig. \ref{fig:asyn}.

\paragraph{Membrane Potential Clipping.} We evaluate the performance of spike-based models under the membrane potential clipping scheme, as shown in Table~\ref{tab:2}. Compared to SDLLM with symmetric quantization, applying a clipping threshold at the 60\% quantile (\(q = 0.6\)) reduces the spike firing rate in the QKV layer from 0.22 to 0.10, leading to  to a 2.2$\times$ reduction in FLOPs and energy consumption. Compared to SpikeLLM,  the QKV layer of SDLLM reduces energy consumption by 15$\times$ while maintaining comparable accuracy.

\paragraph{Comparison with Low-Bit Numerical Inference Models.}
Given that INT4 quantization models are widely regarded as promising solutions for resource-constrained inference, we compare SDLLM with several representative conventional low-bit numerical inference models, including SmoothQuant \citep{xiao2023smoothquant}, OS + \citep{wei2023outlier}, OmniQuant \citep{shao2023omniquant}, AffineQuant \citep{ma2024affinequant}, QLLM \citep{liu2024qllm}, Atom \citep{zhao2023atom}, DuQuant \citep{lin2024duquant}. As shown in Tables~\ref{tab:3} and~\ref{tab:4}, on zero-shot question answering tasks, SDLLM consistently outperforms DuQuant on LLaMA-2-7B, LLaMA-2-13B, and LLaMA3-8B models, achieving state-of-the-art (SOTA) performance on these models. As reported in Table~\ref{tab:app_llama_ppl}, SDLLM also attains lower perplexity (PPL) than QuaRot on the Wikitext2~\citep{merity2016pointer} and C4 \citep{raffel2020t5} datasets in language modeling tasks. 
Meanwhile, compared with the INT4 baseline, SDLLM further reduces FLOPs by $1.2\times$, and reduces energy consumption by $6\times$. Different from low-bit quantization models in edge scenarios, SDLLM presents a potential solution from the SNN perspective. By addressing the insufficient spike representation in SNNs, it achieves INT4 SOTA-level task performance while utilizing the efficient characteristics of spike-driven approaches.

\begin{table*}
\setcaptionwidth{1} 
\caption{Evaluation of Zero-shot QA (↑) results of LLaMA3-8B under DuQuant settings.} 
\label{tab:4}
\setlength{\tabcolsep}{3pt}
\vspace{0.5em}
\begin{adjustbox}{width=\linewidth} 
\begin{tabular}{l c c c c c c c c c c c c@{\hspace{3pt}}c}
\toprule
Model & Spike & Bit & $T \times D$ & $R$ & PIQA & ARC-e & ARC-c & BoolQ & HellaS & WinoG & Avg. & FLOPs(T)  & Power(J)\\
\midrule
LLaMA3-8B & \ding{55}  & - & - & -  & 80.85 & 77.78 & 53.41 & 81.28 & 79.16 & 72.84 & 74.22 & 7.97 & 36.67  \\
SmoothQuant    & \ding{55} & 4-4 & - & - & 54.57 & 31.9  & 24.23 & 52.72 & 31.26 & 51.14 & 40.97 & 1.00 & 4.58   \\ 
OmniQuant      & \ding{55} & 4-4 & - & - & 50.22 & 26.94 & 24.57 & 37.98 & 26.55 & 50.20 & 36.08 & 1.00 & 4.58   \\ 
AffineQuant    & \ding{55} & 4-4 & - & - & 50.71 & 25.93 & 26.02 & 40.55 & 26.07 & 48.46 & 36.29 & 1.00 & 4.58   \\ 
Atom           & \ding{55} & 4-4 & - & - & 62.95 & 49.45 & 30.12 & 60.31 & 53.75 & 56.04 & 52.10 & 1.00 & 4.58   \\ 
DuQuant        & \ding{55} & 4-4 & - & - & 75.68 & 68.48 & 41.81 & 71.99 & 73.07 & 66.22 & 66.21 & 1.00 & 4.58   \\
\rowcolor{lightpink}
SDLLM          & \ding{51} & 4-1.58 & 1 $\times$ 8 & 0.210 & 75.90 &	67.05 	&44.37 	&72.45 	&73.26 &67.01 &	\textbf{66.67} &\textbf{0.84} 	&\textbf{0.75} 
 \\
\bottomrule
\end{tabular}
\end{adjustbox}
\centering
\setcaptionwidth{1} 
\caption{Ablation study of adjusting and unfolded time steps to improve model task performance.}
\label{tab:6}
\setlength{\tabcolsep}{3pt}
\vspace{0.5em}
\begin{adjustbox}{width=\linewidth} 
\begin{tabular}{l c c c c c c c c c c c c@{\hspace{3pt}}c}
\toprule
Model & Spike & Bit & $T \times D$ & $R$ & PIQA & ARC-e & ARC-c & BoolQ & HellaS & WinoG & Avg. & FLOPs(T)  & Power(J)\\
\midrule
LLaMA2-7B  & \ding{55}  & - & -                        & -          & 78.84 & 74.54 & 46.33 & 77.74 & 75.97 & 69.22 & 70.44 & 6.91 & 31.77    \\
SDLLM            & \ding{51} & 4-1.58 & 1 $\times$ 8 & 0.216 & 75.84 & 69.65 & 41.21 & 74.01 & 71.75 & 66.14 & 66.43 & 0.75 & 0.67 \\ 	 
SDLLM            & \ding{51} & 4-1.58 & 1 $\times$ 8.3 & 0.216 & 77.31 & 70.29 & 41.13 & 72.42 & 73.05 & 67.64 & 66.97 & 0.77 & 0.70 \\
SDLLM            & \ding{51} & 4-1.58 & 1 $\times$ 16  & 0.221 & 78.02 & 72.05 & 44.28 & 75.87 & 74.49 & 68.11 & 68.80 & 1.53 & 1.37 \\
\midrule
LLaMA2-13B & \ding{55}  & - & -                       & -          & 80.63 & 77.48 & 49.23 & 80.73 & 79.37 &71.74  & 80.69 &  13.42 & 61.74 \\
SDLLM            & \ding{51} & 4-1.58 & 1 $\times$ 8 & 0.209 & 78.51 & 74.12 & 46.16 & 78.26 & 76.36 & 69.85 & 70.54 & 1.40 & 1.26 \\
SDLLM            & \ding{51} & 4-1.58 & 1 $\times$ 8.3 & 0.209 & 79.33 & 73.99 & 47.70 & 77.09 & 76.94 & 69.85 & 70.82 &  1.46 & 1.31 \\ 
SDLLM            & \ding{51} & 4-1.58 & 1 $\times$ 16 & 0.215 & 80.25 & 76.77 & 49.40 & 77.49 & 77.92 & 69.77 & 71.93 &  2.89 & 2.60 \\
\bottomrule
\end{tabular}
\end{adjustbox}
\vspace{-12pt}
\end{table*}

\begin{table}
\centering
\caption{Ablation study of $\gamma$-SQP on PPL (↓) metrics for Wikitext2 and C4. SDLLM uses $ \mathbf X\,\mathrm{Diag}(\gamma)\mathbf Q$ semantics, while QuaRot uses $\mathbf Q^\top \mathbf W\,\mathrm{Diag}(\gamma)$ semantics.}
\label{tab:gamma-SQP_ppl}
\setlength{\tabcolsep}{3pt}
\begin{adjustbox}{width=1\linewidth} 
\begin{tabular}{lccccccccccccccc}
\toprule
\multirow{2}{*}{Model}  & \multirow{2}{*}{Spike} & \multirow{2}{*}{QOp}  & \multicolumn{2}{c}{LLaMA2-7B}  & \multicolumn{2}{c}{LLaMA2-13B} \\
\cmidrule(lr){4-5}
\cmidrule(lr){6-7}
&  &  & Wiki & C4 & Wiki & C4\\
\midrule
FP16 & \ding{55} & - & 5.47 & 7.26 & 4.88 & 6.73 \\
\midrule
\multirow{4}{*}{QuaRot} & \multirow{4}{*}{\ding{55}} & QKV & 6.37 & 8.62 & 5.32 & 7.39 \\
 & & QKV+Att & 6.71 & 9.09 & 5.53 & 7.69 \\
 & & QKV+Att+FFN & 8.73 & 12.27 & 6.31 & 9.02 \\
 \midrule
\multirow{4}{*}{SDLLM} & \multirow{4}{*}{\ding{51}} & QKV  & 5.61 & 7.40 & 4.96 & 6.80 \\
 & & QKV+Att & 5.77 & 7.61 & 5.09 & 6.95 \\
 & & QKV+Att+FFN & 6.40 & 8.58 & 5.48 & 7.59 \\
\bottomrule
\end{tabular}
\end{adjustbox}
\end{table}

\begin{table}
\centering
\caption{Ablation study of the inference cost with or without using sparse components to control spike firing rate.}
\label{tab:app_ablation}
\setlength{\tabcolsep}{3pt}
\begin{adjustbox}{width=1\linewidth} 
\begin{tabular}{lccccccccccccccc}
\toprule
Model  & Spike  & Bit  &  $T \times D$ & $R$  & FLOPs(T)  & Power(J)\\
\midrule
LLaMA2-7B  & \ding{55}  & - & -                        & -              & 6.91 & 31.77    \\
SDLLM-step1 
                 & \ding{51} & 4-4 & 1 $\times$ 1                      & -   & 0.86 & 3.97 \\
SDLLM-step2             & \ding{51} & 4-1.58 & 1 $\times$ 8  & 0.216   & 0.75 & 0.67 \\
SDLLM-step2               & \ding{51} & 4-1 & 1 $\times$ 15 & 0.500  & 1.62 & 1.46 \\ 
SDLLM-step1  
                 & \ding{51} & 6-6 & 1 $\times$ 1                      & -   &1.94  & 8.94 \\
SDLLM-step2              & \ding{51} & 6-1.58 & 1 $\times$ 32 & 0.222   & 4.60 & 4.14 \\
SDLLM-step2               & \ding{51} & 6-1 & 1 $\times$ 63 & 0.500   &10.20 & 9.18 \\
\midrule
LLaMA2-13B & \ding{55}  & - & -                       & -             & 13.42 & 61.74 \\
SDLLM-step1  
                 & \ding{51} & 4-4 & 1 $\times$ 1 & -   & 1.68  & 7.72 \\
SDLLM-step2               & \ding{51} & 4-1.58 & 1 $\times$ 8 & 0.209  & 1.40  & 1.26 \\
SDLLM-step2                & \ding{51} & 4-1 & 1 $\times$ 15 & 0.500   & 3.15  & 2.83 \\
SDLLM-step1  
                 & \ding{51} & 6-6 & 1 $\times$ 1 & -   & 3.77  & 17.36  \\
SDLLM-step2               & \ding{51} & 6-1.58 & 1 $\times$ 32  & 0.217   & 8.74  & 7.86 \\
SDLLM-step2               & \ding{51} & 6-1 & 1 $\times$ 63 & 0.500   & 19.82 & 17.84 \\
\bottomrule
\end{tabular}
\end{adjustbox}
\end{table}

\begin{table*}
\renewcommand{\arraystretch}{0.88}
\centering
\setcaptionwidth{1} 
\caption{Additional Evaluation of Zero-shot QA (↑) results of LLaMA2-7B and 13B with 6-bit weights under QLLM settings.}
\label{tab:sup_1}
\setlength{\tabcolsep}{3pt}
\vspace{0.5em}
\begin{adjustbox}{width=\linewidth} 
\begin{tabular}{l c c c c c c c c c c c c@{\hspace{3pt}}c}
\toprule
Model & Spike & Bit & $T \times D$ & $R$ & PIQA & ARC-e & ARC-c & BoolQ & HellaS & WinoG & Avg. & FLOPs(T)  & Power(J)\\
\midrule
LLAMA-2-7B  & \ding{55}  & - & - & - & 76.88 & 53.54 & 40.53 & 71.13 & 72.96 & 67.25 & 63.72  & 6.91 & 31.77      \\
SmoothQuant      & \ding{55} & 6-6 & - & - & 75.57 & 53.62 & 39.93 & 69.54 & 71.76 & 66.14 & 62.76 & 1.94 & 8.94 \\
OS+              & \ding{55} & 6-6 & - & - & 76.22 & 52.74 & 40.70 & -     & 71.89 & 65.19 & 61.35 & 1.94 & 8.94  \\
OmniQuant        & \ding{55} & 6-6 & - & - & 76.55 & 53.83 & 40.96 & 68.75 & 55.89 & 65.59 & 60.26 & 1.94 & 8.94 \\
QLLM             & \ding{55} & 6-6 & - & - & 77.48 & 52.99 & 39.33 & -     & 71.38 & 65.98 & 61.43 & 1.94 & 8.94 \\
DuQuant          & \ding{55} & 6-6 & - & - & 76.99 & 52.99 & 40.87 & 70.40 & 72.49 & 67.32 & 63.51 & 1.94 & 8.94 \\
\rowcolor{lightpink}
SDLLM            & \ding{51} & 6-1.58 & 1 $\times$ 32 
                                         & 0.222 & 76.99  & 53.75 & 41.04 & 70.64 & 72.84 &  67.25 & \textbf{63.75}  &4.60 &\textbf{4.14}\\    
\midrule

LLAMA2-13B  & \ding{55} & -  & - & - & 79.05 & 57.91 & 44.20 & 69.02 & 76.60 & 69.69 & 66.08 & 13.42& 65.77 \\
SmoothQuant      & \ding{55} & 6-6 & - & - & 78.29 & 57.41 & 43.86 & 69.50 & 75.02 & 66.93 & 65.17  & 3.77 & 17.36 \\
OS+              & \ding{55} & 6-6 & - & - & 78.29 & 59.13 & 43.34 & -     & 75.37 & 67.56 & 64.74  & 3.77 & 17.36 \\
OmniQuant        & \ding{55} & 6-6 & - & - & 78.24 & 57.58 & 43.86 & 71.10 & 75.52 & 68.35 & 65.78  & 3.77 & 17.36 \\
AffineQuant      & \ding{55} & 6-6 & - & - & 78.35 & 57.58 & 43.34 & 66.73 & 74.71 & 68.59 & 64.88  & 3.77 & 17.36 \\
QLLM             & \ding{55} & 6-6 & - & - & 78.78 & 58.29 & 43.77 & -     & 75.10 & 68.43 & 64.87  & 3.77 & 17.36 \\
DuQuant          & \ding{55} & 6-6 & - & - & 78.62 & 56.94 & 43.43 & 68.35 & 76.19 & 69.22 & 65.46  & 3.77 & 17.36 \\
\rowcolor{lightpink}
SDLLM            & \ding{51} & 6-1.58 & 1 $\times$ 32
                                         & 0.217 & 79.05 & 57.66 & 44.20 & 67.83 & 76.42 & 69.93 & \textbf{65.85}  &8.74 &\textbf{7.86}\\

\bottomrule
\end{tabular}
\end{adjustbox}

\vspace{4pt}
\setcaptionwidth{1} 

\caption{Additional Evaluation of Zero-shot QA (↑) results of LLaMA3-8B with 6-bit weights under DuQuant settings.} 
\label{tab:sup_2}
\setlength{\tabcolsep}{3pt}
\vspace{0.5em}
\begin{adjustbox}{width=\linewidth} 
\begin{tabular}{l c c c c c c c c c c c c@{\hspace{3pt}}c}
\toprule
Model & Spike & Bit & $T \times D$ & $R$ & PIQA & ARC-e & ARC-c & BoolQ & HellaS & WinoG & Avg. & FLOPs(T) & Power(J)\\
\midrule
LLaMA3-8B        & \ding{55} & - & - & - & 80.85 & 77.78 & 53.41 & 81.28 & 79.16 & 72.84 & 74.22 & 7.97 & 36.67 
\\
SmoothQuant    & \ding{55} & 6-6 & - & -  & 78.94 & 75.88 & 49.49 & 77.58 & 77.39 & 70.80 & 71.68 & 2.24 & 10.31  \\
OmniQuant      & \ding{55} & 6-6  & - & - & 78.90 & 73.95 & 47.35 & 74.95 & 76.77 & 70.56 & 70.41 & 2.24 &	10.31 \\ 
AffineQuant    & \ding{55} & 6-6 & - & -  & 78.73 & 73.32 & 46.08 & 74.59 & 77.08 & 70.88 & 70.11 & 2.24 &	10.31 \\ 

DuQuant      & \ding{55} & 6-6 & - & -  & 80.20 & 77.27 & 52.05 & 80.12 & 79.14 & 72.77 & 73.59 & 2.24 &	10.31  \\
\rowcolor{lightpink}
SDLLM  & \ding{51} & 6-1.58 & 1 $\times$ 32 & 0.213  & 80.20& 77.23 & 52.22& 82.05& 79.01& 73.56 & \textbf{74.04} & 5.09 & \textbf{4.58} 
 \\
\bottomrule
\end{tabular}
\end{adjustbox}
\end{table*}

\paragraph{Evaluation on Qwen2.5-14B.}
We further evaluate SDLLM on the Qwen2.5-14B model (Tables~\ref{tab:app_zero_shot_qwen} and~\ref{tab:app_ppl_qwen}). 
On zero-shot question answering tasks, SDLLM achieves performance close to the FP16 baseline of Qwen2.5-14B, while maintaining lower perplexity on Wikitext2 and C4. 
Compared with the INT4 baseline, SDLLM reduces FLOPs and energy consumption by $1.2\times$, $1.2\times$, and $6\times$, respectively.

\paragraph{Evaluation on More Complex Language Tasks.}
As shown in Table~7, we further evaluate SDLLM on a broader and more challenging range of tasks beyond commonsense question answering and reading comprehension (BoolQ), including mathematical reasoning (GSM8K), reading comprehension (SQuAD), and world knowledge question answering (TriviaQA), on the LLaMA-2-7B, LLaMA-2-13B, and LLaMA3-8B models. 
Overall, SDLLM demonstrates improved performance over the QuaRot baseline, while retaining its advantages in FLOPs and energy consumption. 
These results indicate that, in addition to commonsense reasoning QA tasks, SDLLM also excels in a broader and more challenging range of tasks, maintaining a favorable balance between task performance and efficiency as task difficulty increases.

\subsection{Ablation Study}
\paragraph{Performance Improvement.} As shown in Table \ref{tab:gamma-SQP_ppl}, we demonstrate the PPL metrics for different quantization operators influenced by the $\gamma$-SQP on the Wikitext2 and C4 datasets. SDLLM with $\gamma$-SQP significantly reduces PPL compared to the QuaRot baseline, both based on RTN weights. Additionally, as shown in Table \ref{tab:1}, SDLLM improves task accuracy by 6.72\% and 4.93\% on LLaMA2-7B and LLaMA2-13B, respectively, while reducing  FLOPs by 1.2× and energy consumption by 6×.
As shown in Table \ref{tab:6}, for SDLLM, when the time step \( D \) is slightly increased from 8 to 8.3 (corresponding to a KV cache of \( D = 16 \)), the accuracy of LLaMA-2-7B increases from 66.43\% to 66.97\%, and the accuracy of LLaMA-2-13B increases from 70.54\% to 70.82\%. In this configuration, compared to the INT4 baseline, LLaMA-2-7B achieves a 1.1× reduction in FLOPs, and a 6× reduction in energy consumption. Further increasing the time step \( D \) to 16 boosts LLaMA-2-7B's accuracy to 68.80\% and LLaMA-2-13B's accuracy to 71.93\%, while energy consumption is still reduced by 3×.

\paragraph{Sparse Optimization.} As shown in Table \ref{tab:app_ablation}, compared to the configuration without sparse optimization, SDLLM significantly reduces the model's spike firing rate from 0.5 to 0.2 through sparse optimization. Additionally, the time step \( D \) is reduced from 15 to 8, halving the temporal cost. Although in the asynchronous firing mode, the unfolded time step \( D \) does not affect latency, increasing \( D \) may increase the number of operations, leading to higher energy consumption. With sparse optimization, SDLLM reduces FLOPs and energy consumption by 2.2×, respectively.

\paragraph{6-bit Weights.} In SDLLM, we also perform weight quantization. Compared to the 4-bit weights used in the main experiments, when using 6-bit weights and unfolding the time step \( D \), as shown in Tables \ref{tab:sup_1} and \ref{tab:sup_2}, SDLLM achieves SOTA performance with INT6 quantization. Clearly, the LLaMA and Qwen series models perform better in general knowledge question answering tasks and more widely and challenging tasks, as well as perplexity, compared to the 4-bit weight models. However, increasing the weight bit-width and unfolding the time step \( D \) reduce the energy efficiency advantage of SDLLM, so we do not consider it as the primary approach. Nevertheless, compared to the INT6 baseline, SDLLM still reduces energy consumption by 2×.

\section{Discussion}
The goal of this work is to draw inspiration from the brain's information processing mechanism and explore how to effectively integrate spiking-driven characteristics into the inference process of large-scale models. Existing spike encoding methods primarily focus on vision or small-scale tasks, and struggle to meet the demands of large language models (LLMs) in terms of both representational capacity and computational efficiency. To address this issue, we propose a plug-and-play two-step spike encoding method based on $\gamma$-SQP. The $\gamma$-SQP optimizes the quantization process by aligning it with the semantic space learned from the training model, reducing the concentration of quantization errors and effectively improving the accuracy of spike encoding.

On this basis, we further focus on the sparsity issue of spike encoding in LLMs, as the spike firing rate directly determines the efficiency of the spiking-driven mechanism. By analyzing the relationship between membrane potential, spike count, and spike trains in the two-step spike encoding process, we reveal the patterns and rules of spike firing. We introduce two sparsification mechanisms: bidirectional encoding and membrane potential clipping. These mechanisms effectively reduce the frequency of high-count spikes, increase the occurrence of zero-count spikes, and halve the time expansion length while maintaining representational capacity, significantly lowering the model's spike firing rate and improving the efficiency of spiking-driven inference. This method demonstrates good generalizability, able to adapt to models and tasks of different scales, with competitive performance and efficiency. It provides a new perspective for exploring edge scenarios in LLMs from the SNN perspective.

Spiking-driven models naturally align with the low-power computing paradigm in terms of their computational structure. In the context of the neuromorphic spiking paradigm, we discuss the implementation of SDLLM on neuromorphic chips. The $T \times D$ time-step paradigm used by SDLLM supports both synchronous and asynchronous spike emission modes. In asynchronous emission mode, since there is no global clock, the $D$ dimension can be emitted in a very short time. However, when the $D$ dimension becomes too large, it may increase power consumption and communication bandwidth requirements on the chip, a problem that can be optimized by adjusting the concurrency of $D$.
Neuromorphic hardware typically updates neuron states through event-triggered mechanisms and features locally stored synaptic states. {Meanwhile, the asynchronous firing mode enables locally stored synaptic weights to be shared across multiple unfolded timesteps.} The transmission of spike signals triggers the computation process, providing an efficient implementation foundation for sparse, on-demand models. SDLLM does not rely on the training process of SNNs but is plug-and-play during inference, with the core operator influencing its performance being the sparse spiking-driven operator with no matrix multiplication.
As neuromorphic chips continue to evolve in programmability and representational flexibility, existing chips are already able to express different computational semantics through configurable neuron models and hierarchical spike events. This enables neuromorphic chips to effectively carry SDLLM's spiking representations. As a soft-hardware co-design paradigm, SDLLM's exploration of LLM-level spike encoding for edge scenarios provides insights for the soft-hardware co-design of next-generation neuromorphic chips.

\section{Conclusion}
In this work, we present the first spike-driven LLM that eliminates matrix multiplication entirely by leveraging sparse addition, built upon multiple LLM architectures, addressing the issues of insufficient spiking representation and sparsity at the LLM level.
In addition to comparing SNNs with full-precision ANNs, we also systematically benchmark SNNs against low-bit solutions in edge computing scenarios to assess their potential as a spiking-driven paradigm.
Our results show that, compared to low-bit edge solutions, SDLLM achieves competitive accuracy while reducing energy consumption by up to 13 $\times$.
This work provides the first compelling evidence that SNNs are not only feasible for LLMs but also have the potential to rival low-bit edge solutions in both accuracy and energy efficiency, laying a critical foundation for the next generation of neuromorphic general intelligence.

\section*{Acknowledgement}
This work was partially supported by CAS Project for Young Scientists in Basic Research (YSBR-116), National Natural Science Foundation of China (62325603, 62236009, U22A20103), Beijing Science and Technology Plan (Z241100004224011), and Shanghai NeuHelium Neuromorphic Technology Co., Ltd.

\bibliographystyle{harv}
\bibliography{refs}

\clearpage
\section*{Appendix} 
\setcounter{section}{0} 
\renewcommand{\thesection}{\Alph{section}}
\setcounter{subsection}{0}
\renewcommand{\thesubsection}{\Alph{section}.\arabic{subsection}} 
\setcounter{table}{0}
\setcounter{figure}{0}
\renewcommand{\thetable}{A\arabic{table}}
\renewcommand{\thefigure}{A\arabic{figure}} 

\section{Details of Operations and Energy Consumption}
\label{app:details_of_op_ec}
\subsection{FLOPs}
We refer to the FLOPs calculation method for \( q \)-bit operations from the Q-DETR paper. For 2-bit, 3-bit, and 4-bit operations, the FLOPs for 2-bit operations is \( \frac{1}{32} \) of the 32-bit FLOPs, for 3-bit operations it's \( \frac{1}{16} \), and for 4-bit operations it's \( \frac{1}{8} \) \citep{xu2023qdetr,liu2020birealnet}, since the current CPU can parallelize bitwise XNOR and popcount operations.

For higher-bit operations, they can be decomposed into multiple lower-bit operations. For example, 4-bit $\times$ 4-bit is decomposed into $2 \times 2$ operations of 2-bit $\times$ 2-bit, corresponding to $\frac{1}{8}$ FLOPs. We calculate $q$-bit FLOPs following the above principle for all methods considered in this paper, including all baseline and comparison methods. For SNNs, FLOPs$_{\text{spike}}$ = FLOPs$_{q\text{-bit}} \times T \times R$ \citep{xu2025neurormorphic}.

\subsection{Power}
When evaluating algorithmic efficiency, prior work in the SNN community typically abstracts away specific hardware implementation details and analyzes models using theoretical energy consumption metrics \citep{qiu2025quantized,yao2024spikev2,wang2023riformer,yang2022lead,panda2020toward}. Such theoretical estimates are primarily employed to facilitate qualitative energy comparisons between different SNN and ANN algorithms, with the aim of examining energy-efficiency trends arising from factors such as operation counts, sparsity, and bit-width configurations.

We estimate the theoretical energy consumption layer-wise, taking into account that different layers exhibit different firing rates. Specifically, we compute the FLOPs of each layer separately and sum them to obtain the overall computational cost.
For ANN models, the energy cost is estimated as
$
\text{Power} = \text{FLOPs} \times E_{\text{MAC}},
$
whereas for SNN models, the energy cost is estimated as
$
\text{Power} = \text{FLOPs} \times E_{\text{AC}}.
$
Under a 45\,nm technology node with a 32-bit floating-point implementation, the energy cost of a multiply--accumulate (MAC) operation is $E_{\text{MAC}} = 4.6$\,pJ, while that of an accumulate (AC) operation is $E_{\text{AC}} = 0.9$\,pJ~\citep{yao2025scaling,luo2024integer,lv2023spikebert}.

From a system implementation perspective, the energy characteristics of practical systems depend on a combination of data movement, computation, and control logic overhead. Unlike traditional non-neuromorphic systems, the potential energy advantage of neuromorphic systems is typically centered on sparse computation, with execution relying on event-driven logic and locally stored synaptic weights, resulting in fundamentally different organizations of computation and data access compared to conventional systems.
Current neuromorphic chip architectures have, to some extent, explored native support for various spike representations and operator-level primitives, and the evolution of these architectural features is often driven by continuous algorithmic exploration of new computational paradigms in SNN research. The proposed SDLLM model demonstrates that sparse event-stream modeling and temporal encoding can coexist with high accuracy in large-scale language model tasks, thereby offering algorithm-level insights for future event-driven neuromorphic chip design.

\begin{figure}
    \centering
    \includegraphics[width=1\linewidth]{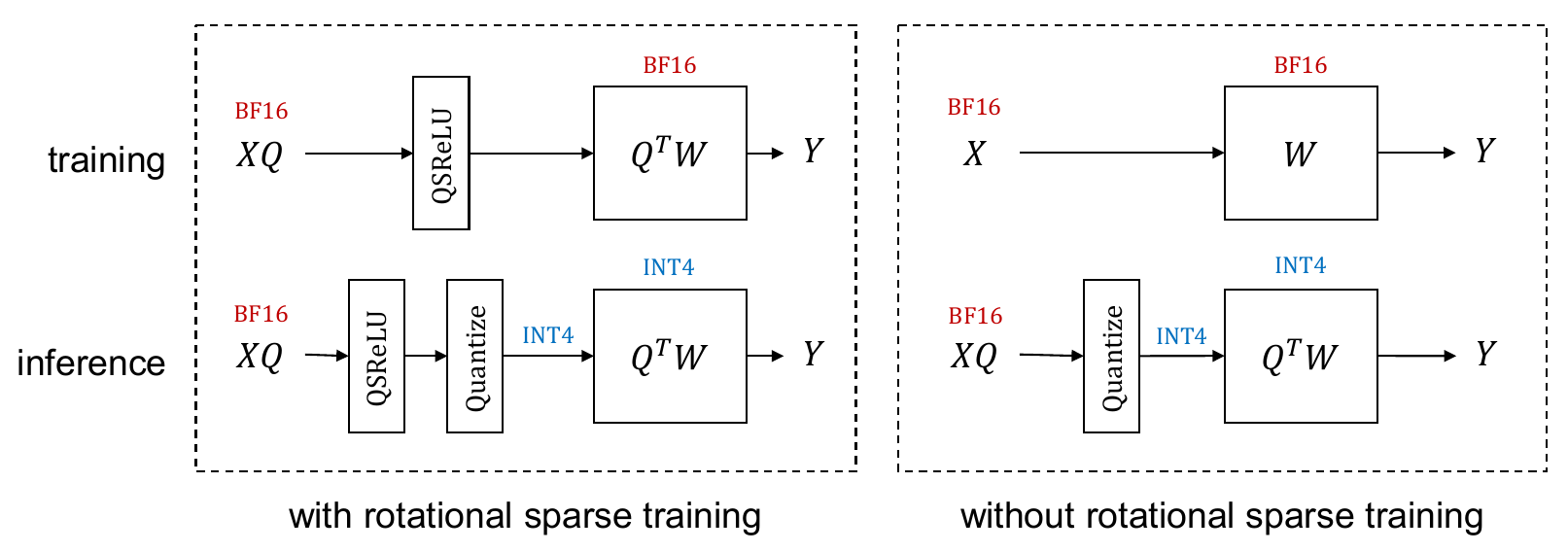}
    \caption{
    Implementation of rotational sparse training for enhancing spike sparsity.
    }
    \label{fig:sup_c}
\end{figure}

\section{Rotational Sparse Training}
\label{app:sparse_training}
\subsection{Training setup}

Following the setup in {ReLU Strikes Back} \cite{mirzadehrelu}, we fine-tune the pre-trained LLaMA series pre-trained models on the RefinedWeb dataset \citep{penedo2023refinedweb} to evaluate the performance of SDLLM under the membrane potential clipping method. {We train on 8 A800 GPUs with approximately 10 million tokens and use the AdamW optimizer with a fixed learning rate of $1.5 \times 10^{-5}$.} To improve training efficiency and reduce memory consumption, we adopt the ZeRO Stage 2 optimization strategy \citep{rajbhandari2020zero} provided by DeepSpeed for distributed management of optimizer states and gradients.

\subsection{Training Strategy}
As discussed in the \textit{Joint Sparsity and Rotation Matrices} subsection of Section~\ref{sec:4.4}, we adopt a rotational sparse training strategy to enhance quantization performance and activation sparsity during training. Specifically, during training, as illustrated in Fig.~\ref{fig:sup_c}, we apply an orthogonal rotation matrix \( Q \) only to the linear operators whose outputs are involved in sparsification, i.e., those followed by the Quantile-Shifted ReLU activation function. This transformation improves the uniformity of feature distributions and facilitates effective sparsity learning. For operators not participating in sparsification, no rotation is applied during training, thereby avoiding unnecessary computational overhead. During inference, however, we apply the rotation matrix \( Q \) uniformly to all linear operators and use the rotated weights \( Q^T W \) to ensure compatibility across both sparse and non-sparse computation paths. This strategy strikes a balance between training efficiency and inference consistency, demonstrating the practicality and generalizability of rotational sparse training.

\section{Spike Firing Details}
As mentioned earlier in Section~\ref{sec:4.1}, the computational cost of non-matrix multiplication operators is several orders of magnitude lower. Therefore, in Tab.~\ref{tab:s3}, we present the spike firing behavior and corresponding FLOPs of the linear layers in the SDLLM based on the LLaMA-2 7B baseline. In all tables, $N$ denotes the train length, and we uniformly set $N = 1024$. In addition, $D_h$ and $D_i$ represent the hidden size and intermediate size, respectively.
In addition to applying spiking to the linear layers, we also spiked the KV Cache, similar to how quantization methods process the KV Cache. The spiked KV Cache is directly involved in the computation of spiking attention.
Additionally, we visualize the firing of some neurons in Fig. \ref{fig:sup_111}.

\newpage

\begin{figure*}
    \centering
     \begin{minipage}[b]{0.28\textwidth}
        \centering
        \includegraphics[width=\textwidth]{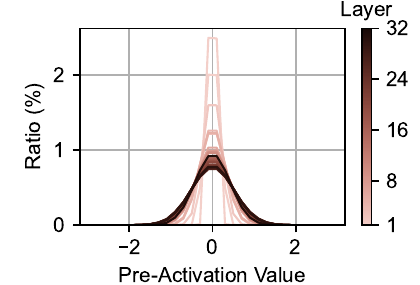}
        \scriptsize {QKV (Pre Act)}
        \label{fig:sup_d1}
    \end{minipage}
    \begin{minipage}[b]{0.36\textwidth}
        \centering
        \includegraphics[width=\textwidth]{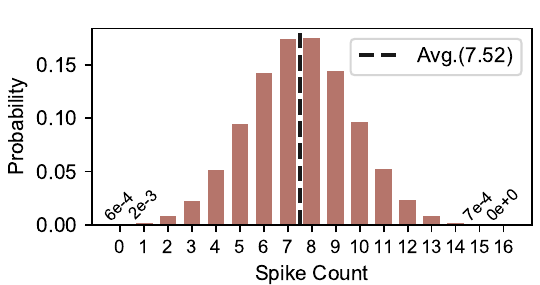}
        \scriptsize {QKV (Symmetric)}
        \label{fig:sup_d2}
    \end{minipage}
    \begin{minipage}[b]{0.24\textwidth}
        \centering
        \includegraphics[width=\textwidth]{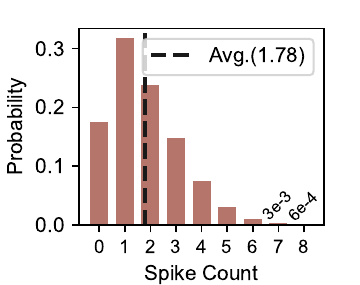}
        \scriptsize {QKV (Asymmetric)}
        \label{fig:sup_d3}
    \end{minipage}
    \begin{minipage}[b]{0.28\textwidth}
        \centering
        \includegraphics[width=\textwidth]{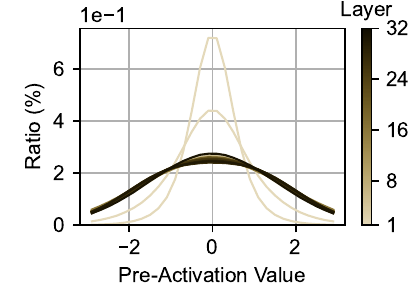}
        \scriptsize {K Cache (Pre Act)}
        \label{fig:sup_d4}
    \end{minipage}
    \begin{minipage}[b]{0.36\textwidth}
        \centering
        \includegraphics[width=\textwidth]{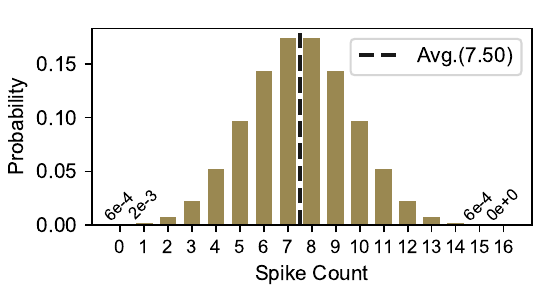}
        \scriptsize {K Cache (Symmetric)}
        \label{fig:sup_d5}
    \end{minipage}
    \begin{minipage}[b]{0.24\textwidth}
        \centering
        \includegraphics[width=\textwidth]{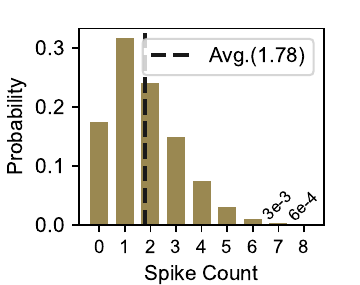}
        \scriptsize {K Cache (Asymmetric)}
        \label{fig:sup_d6}
    \end{minipage}
    \begin{minipage}[b]{0.28\textwidth}
        \centering
        \includegraphics[width=\textwidth]{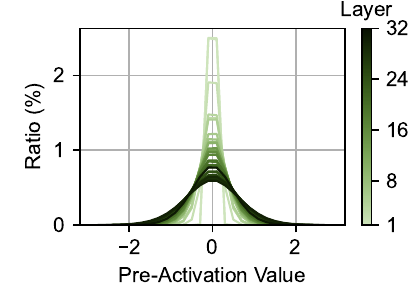}
        \scriptsize {V Cache (Pre Act)}
        \label{fig:sup_d7}
    \end{minipage}
    \begin{minipage}[b]{0.36\textwidth}
        \centering
        \includegraphics[width=\textwidth]{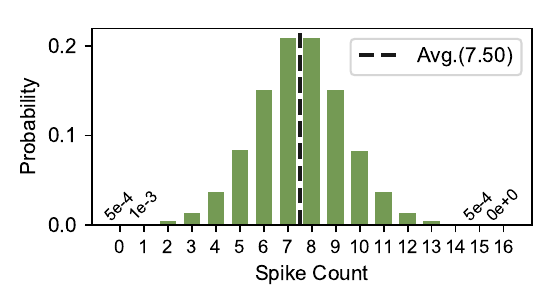}
        \scriptsize {V Cache (Symmetric)}
        \label{fig:sup_d8}
    \end{minipage}
    \begin{minipage}[b]{0.24\textwidth}
        \centering
        \includegraphics[width=\textwidth]{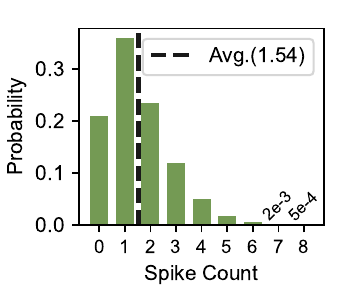}
        \scriptsize {V Cache (Asymmetric)}
        \label{fig:sup_d9}
    \end{minipage}
    \begin{minipage}[b]{0.28\textwidth}
        \centering
        \includegraphics[width=\textwidth]{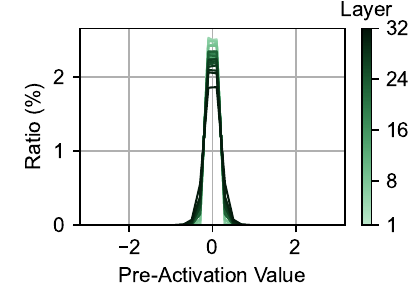}
        \scriptsize {O (Pre Act)}
        \label{fig:sup_d10}
    \end{minipage}
    \begin{minipage}[b]{0.36\textwidth}
        \centering
        \includegraphics[width=\textwidth]{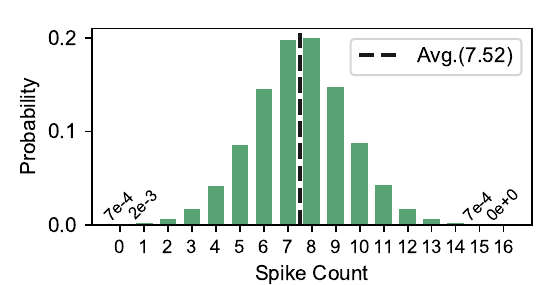}
        \scriptsize {O (Symmetric)}
        \label{fig:sup_d11}
    \end{minipage}
    \begin{minipage}[b]{0.24\textwidth}
        \centering
        \includegraphics[width=\textwidth]{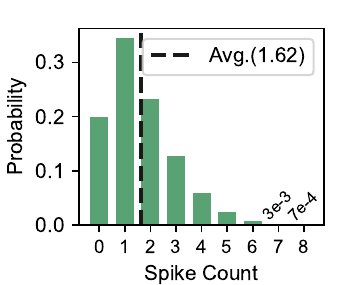}
        \scriptsize {O (Asymmetric)}
        \label{fig:sup_d12}
    \end{minipage}
    \begin{minipage}[b]{0.28\textwidth}
        \centering
        \includegraphics[width=\textwidth]{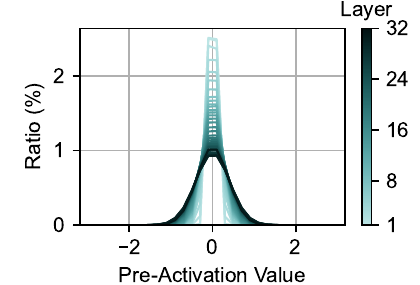}
        \scriptsize {UpGate (Pre Act)}
        \label{fig:sup_d13}
    \end{minipage}
    \begin{minipage}[b]{0.36\textwidth}
        \centering
        \includegraphics[width=\textwidth]{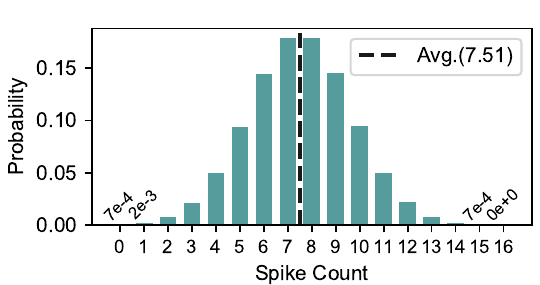}
        \scriptsize {UpGate (Symmetric)}
        \label{fig:sup_d14}
    \end{minipage}
    \begin{minipage}[b]{0.24\textwidth}
        \centering
        \includegraphics[width=\textwidth]{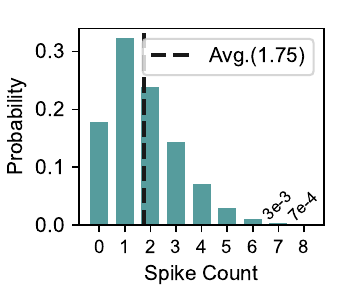}
        \scriptsize {UpGate (Asymmetric)}
        \label{fig:sup_d15}
    \end{minipage}
    \begin{minipage}[b]{0.28\textwidth}
        \centering
        \includegraphics[width=\textwidth]{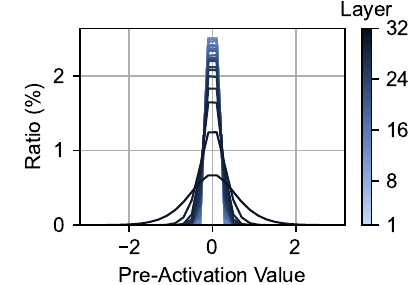}
        \scriptsize {Down (Pre Act)}
        \label{fig:sup_d16}
    \end{minipage}
    \begin{minipage}[b]{0.36\textwidth}
        \centering
        \includegraphics[width=\textwidth]{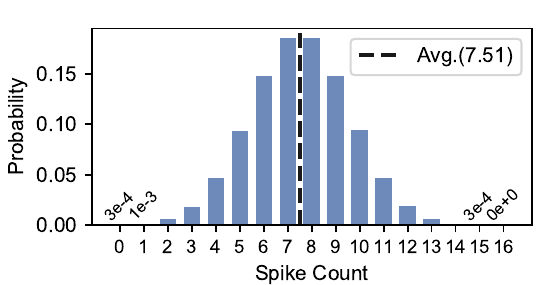}
        \scriptsize {Down (Symmetric)}
        \label{fig:sup_d17}
    \end{minipage}
    \begin{minipage}[b]{0.24\textwidth}
        \centering
        \includegraphics[width=\textwidth]{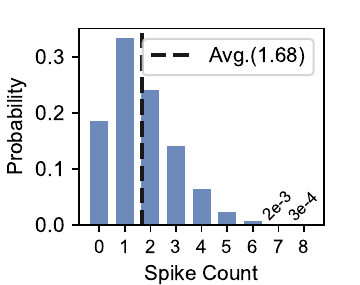}
        \scriptsize {Down (Asymmetric)}
        \label{fig:sup_d18}
    \end{minipage}
    \caption{
   Bidirectional encoding under symmetric quantization consistently reduces 
average spike counts while halving unfolded time steps, dramatically 
lowering firing rate.
    }
    \label{fig:sup_d}
\end{figure*}

\newpage

\begin{table*}
\centering
\setcaptionwidth{0.8} 
\caption{Spike Firing Details and FLOPs of Linear Layers in LLaMA2-7B.}
\label{tab:s3}
\vspace{0.5em}
\begin{adjustbox}{width=0.8\linewidth} 
\begin{tabular}{llccccc}
\toprule
{Model} & {Layer} & {Time Complexity} & $T \times D$ & $R$ & {FLOPs(G)} & {Power(mJ)}  \\
\midrule
\multirow{7}{*}{\shortstack{LLaMA2-7B\\4-1.58}}
 & k\_proj      & $N{D_h}^2$ & 1 $\times$ 8 & 0.2230 & 1.92 &1.73 \\
 & v\_proj      & $N{D_h}^2$ & 1 $\times$ 8 & 0.2230 & 1.92 &1.73 \\
 & q\_proj      & $N{D_h}^2$ & 1 $\times$ 8 & 0.2230 & 1.92 &1.73 \\
 & out\_proj    & $N{D_h}^2$ & 1 $\times$ 8 & 0.2028 & 1.74 &1.57 \\
 & gate\_proj   & $N{D_h}{D_i}$ & 1 $\times$ 8 & 0.2189 & 5.05 &4.55\\
 & up\_proj     & $N{D_h}{D_i}$ & 1 $\times$ 8 & 0.2189 & 5.05 &4.55\\
 & down\_proj   & $N{D_h}{D_i}$ & 1 $\times$ 8 & 0.2096 & 4.84 &4.36 \\
 \midrule
\multirow{7}{*}{\shortstack{LLaMA2-7B\\4-1}}
 & k\_proj      & $N{D_h}^2$ & 1 $\times$ 16 & 0.2257 & 3.88 &3.49 
\\
 & v\_proj      & $N{D_h}^2$ & 1 $\times$ 16 & 0.2257 & 3.88 &3.49 
\\
 & q\_proj      & $N{D_h}^2$ & 1 $\times$ 16 & 0.2257 & 3.88 &3.49 
\\
 & out\_proj    & $N{D_h}^2$ & 1 $\times$ 16 & 0.2192 & 3.77 &3.39 
\\
 & gate\_proj   & $N{D_h}{D_i}$ & 1 $\times$ 16 & 0.2212 & 10.21 &9.19 \\
 & up\_proj     & $N{D_h}{D_i}$& 1 $\times$ 16 & 0.2212 & 10.21 &9.19 
\\
 & down\_proj   & $N{D_h}{D_i}$ & 1 $\times$ 16 & 0.2192 & 10.12 &9.11 
 \\
 \midrule
 \multirow{7}{*}{\shortstack{LLaMA2-7B\\6-1.58}}
 & k\_proj      & $N{D_h}^2$ & 1 $\times$ 32 & 0.2284 & 11.77 &10.59 
\\
 & v\_proj      & $N{D_h}^2$ & 1 $\times$ 32 & 0.2284 & 11.77 &10.59 
\\
 & q\_proj      & $N{D_h}^2$ & 1 $\times$ 32 & 0.2284 & 11.77 &10.59 
\\
 & out\_proj    & $N{D_h}^2$ & 1 $\times$ 32 & 0.2217 & 11.43 &10.29 
\\
 & gate\_proj   & $N{D_h}{D_i}$ & 1 $\times$ 32 & 0.2237 & 30.99 &27.89 
\\
 & up\_proj     & $N{D_h}{D_i}$ & 1 $\times$ 32 & 0.2237 & 30.99 &27.89 
\\
 & down\_proj   & $N{D_h}{D_i}$ & 1 $\times$ 32 & 0.2133 & 29.54 &26.59 
\\

\bottomrule
\end{tabular}
\end{adjustbox}
\end{table*}

\newpage

\begin{figure*}
    \centering
    \includegraphics[width=0.75\linewidth]{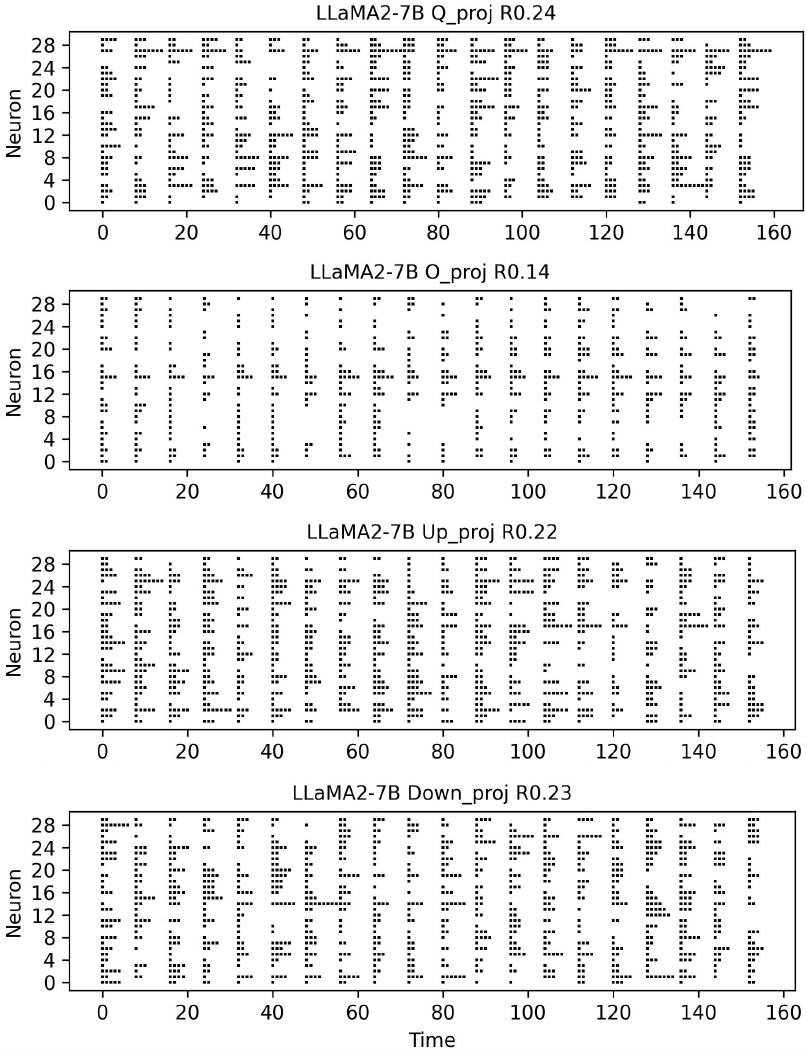}
    \caption{
    SDLLM spike visualization in LLaMA2-7B with single-layer neuron sampling.
    }
     \label{fig:sup_111}
\end{figure*}

\end{sloppypar}
\end{document}